\documentclass{article} % For LaTeX2e
\usepackage{iclr2025_conference,times}

\usepackage{amsmath,amsfonts,bm}
\usepackage{physics}
\usepackage{amsmath}
\usepackage{tikz}
\usepackage{mathdots}
\usepackage{yhmath}
\usepackage{cancel}
\usepackage{color}
\usepackage{siunitx}
\usepackage{array}
\usepackage{multirow}
\usepackage{amssymb}
\usepackage{gensymb}
\usepackage{tabularx}
\usepackage{extarrows}
\usepackage{booktabs}
\usetikzlibrary{fadings}
\usetikzlibrary{patterns}
\usetikzlibrary{shadows.blur}
\usetikzlibrary{shapes}
\def\figref#1{Figure~\ref{#1}}
\def\twofigref#1#2{Figures \ref{#1} and \ref{#2}}
\def\quadfigref#1#2#3#4{Figures \ref{#1}, \ref{#2}, \ref{#3} and \ref{#4}}
\def\secref#1{Section~\ref{#1}}
\def\eqref#1{Equation~\ref{#1}}
\def\algref#1{Algorithm~\ref{#1}}
\def\1{\bm{1}}

\DeclareMathAlphabet{\mathsfit}{\encodingdefault}{\sfdefault}{m}{sl}
\SetMathAlphabet{\mathsfit}{bold}{\encodingdefault}{\sfdefault}{bx}{n}

\newcommand{\E}{\mathbb{E}}

\newcommand{\R}{\mathbb{R}}

\usepackage{hyperref}
\usepackage{url}
\usepackage{booktabs}
\usepackage{tabularray}
\usepackage{tabularx}
\usepackage{amsthm}
\usepackage{amssymb}
\usepackage{xspace}
\usepackage{graphicx}
\usepackage{caption}
\usepackage{subcaption}
\usepackage[ruled]{algorithm2e}
\usepackage{algorithmic}
\usepackage{thmtools, thm-restate}

\newtheorem{proposition}{Proposition}[section]

\theoremstyle{definition}
\newtheorem{definition}{Definition}[section]

\newcommand\ie{\emph{i.e.}}
\newcommand{\methodName}{LOOM-CFM\xspace}
\newcommand{\methodNameExtended}{Looking Out Of Minibatch\xspace}

\title{Faster Inference of Flow-Based Generative Models via Improved Data-Noise Coupling}
\author{%
Aram Davtyan \\
University of Bern \\
Bern, Switzerland \\
\texttt{aram.davtyan@unibe.ch}
\And
Leello Tadesse Dadi \\
EPFL \\
Lausanne, Switzerland \\
\texttt{leello.dadi@epfl.ch} \hspace{2.46cm}
\AND
Volkan Cevher \\
EPFL \\
Lausanne, Switzerland \\
\texttt{volkan.cevher@epfl.ch}
\And
Paolo Favaro \\
University of Bern \\
Bern, Switzerland \\
\texttt{paolo.favaro@unibe.ch}
}

\iclrfinalcopy % Uncomment for camera-ready version, but NOT for submission.
\begin{document}

\maketitle

\begin{abstract}
Conditional Flow Matching (CFM), a simulation-free method for training continuous normalizing flows, provides an efficient alternative to diffusion models for key tasks like image and video generation. The performance of CFM in solving these tasks depends on the way data is coupled with noise. A recent approach uses minibatch optimal transport (OT) to reassign noise-data pairs in each training step to streamline sampling trajectories and thus accelerate inference. However, its optimization is restricted to individual minibatches, limiting its effectiveness on large datasets. To address this shortcoming, we introduce \methodName (\methodNameExtended-CFM), a novel method to extend the scope of minibatch OT by preserving and optimizing these assignments across minibatches over training time. Our approach demonstrates consistent improvements in the sampling speed-quality trade-off across multiple datasets. \methodName also enhances distillation initialization and supports high-resolution synthesis in latent space training.
\end{abstract}

\section{Introduction}
\label{sec:intro}
% The impact of generative models on the current landscape of deep learning cannot be overstated. Beginning with the early advances in adversarial generative networks~\citep{goodfellow2020generative}, variational autoencoders~\citep{kingma2013auto}, and normalizing flows~\citep{papamakarios2021normalizing} and further propelled by the advent of diffusion models~\citep{ho2020denoising, dhariwal2021diffusion} and flow-based approaches~\citep{lipman2022flow, albergo2023stochastic}, the performance of generative AI has reached unprecedented heights in recent years. 
% Diffusion- and flow-based models, which we refer to as iterative methods in this paper, involve the gradual transformation of simple known noise distributions into complex data distributions through a sequence of small steps and have demonstrated remarkable capabilities in generating high-fidelity and diverse content. The high quality and training stability of these methods has led to them rapidly permeating nearly every area of continuous content creation, from image and video generation to motion and audio synthesis and beyond.

The high-quality outputs and training stability of modern diffusion~\citep{ho2020denoising, dhariwal2021diffusion} and flow-based models~\citep{lipman2022flow, albergo2023stochastic}, or more generally \emph{iterative denoising methods}, have led to their rapid adoption in nearly every area of content creation, from image~\citep{esser2024scaling} and video generation~\citep{davtyan2023efficient} to motion~\citep{hu2023motion} and audio synthesis~\citep{guan2024lafma}. 

Despite the superior performance of iterative methods, they require multiple evaluations of the underlying model during inference to generate content. This requirement stems from the gradual transformation of the initial sample into the desired data point, resulting in slower operation compared to single-pass methods like GANs \citep{goodfellow2020generative}. 

To mitigate this drawback, recent research has focused on expediting the generation process through various strategies, including enhanced training techniques \citep{lee2023minimizing, bartosh2024neural}, model distillation \citep{luhman2021knowledge, song2023consistency, liu2023flow, salimans2022progressive}, and sampling modifications \citep{dhariwal2021diffusion, lu2022dpm, shaul2024bespoke}. Specifically, iterative models employing a probability flow ordinary differential equation (ODE) framework are promising, as minimizing the curvature of their generative trajectories can significantly reduce the number of required network evaluations, thereby accelerating sampling.

This paper explores the Conditional Flow Matching (CFM) framework and its derivatives \citep{lipman2022flow, albergo2023stochastic, liu2023flow, tong2023conditional}. CFM provides a straightforward, efficient, and versatile training method that is not dependent on the initial noise distribution. Although data and noise are typically sampled independently, the joint distribution of data-noise pairs, or data-noise coupling, significantly impacts the curvature of sampling trajectories in CFM-trained models \citep{lee2023minimizing}. Sampling these pairs most effectively follows the Optimal Transport (OT) plan \citep{villani2009optimal} between the noise and data distributions \citep{liu2023flow}. However, computing the OT plan for large datasets is not feasible, and while minibatch OT methods offer an approximation during training \citep{tong2023conditional, pooladian2023multisample}, their effectiveness decreases with increasing dataset size.

\if 0
% To address these limitations, we propose a method to improve the effectiveness of minibatch OT-based CFM. The key idea is to store the noise-data assignments and to improve this coupling throughout the training process by reassigning them within each new minibatch. This allows different minibatches to implicitly communicate the local OT assignments with each other and leads to a better approximation of the global OT plan. To prevent overfitting of our model to the fixed noise-data assignments, we additionally propose to allocate multiple noise instances per data point and sample one of those at random at each training step. We call our method \methodName for \methodNameExtended. We provide a convergence analysis of our method and experimentally show the superior performance of \methodName compared to prior work on standard benchmarks. In addition, we show that \methodName provides a better initialization for distillation to further speed up the inference time and is also compatible with latent flow matching for higher resolution generation. To summarize, our contributions are as follows:
\fi 

To overcome these limitations, we introduce a method that enhances the effectiveness of minibatch OT-based CFM. Central to our approach is the preservation and iterative refinement of noise-data pairings within each minibatch, facilitating implicit communication of local OT assignments across different minibatches, thus achieving a more accurate approximation of the global OT plan. To avoid model overfitting to static noise-data assignments, we further propose assigning multiple noise instances to each data point, selecting one randomly during each training step. We refer to our approach as \methodName, standing for \methodNameExtended-CFM. We provide a convergence analysis and demonstrate through experiments that \methodName outperforms existing methods on standard benchmarks. Additionally, \methodName serves as an effective initialization for model distillation, further enhancing inference speed, and is compatible with latent flow matching for generating higher-resolution outputs. Our contributions can be summarized as follows:

\begin{itemize}
    \item We introduce \methodName \--- a novel iterative algorithm to boost the generation speed and accuracy of CFMs by optimizing the global data-noise assignments of minibatch OT; % by locally updating fixed data-noise assignments;
    \item We prevent overfitting of  fixed data-noise assignments at no computational cost by allocating multiple noise samples per data point (which artificially increases  the dataset size);
    \item We present a convergence analysis of \methodName;
    \item We evaluate \methodName extensively and demonstrate its superior performance over prior work. Specifically, \methodName reduces the FID with 12 NFE by 41\% on CIFAR10, 46\% on ImageNet-32, and 54\% on ImageNet-64 compared to minibatch OT methods.
\end{itemize}

\section{Background}
\label{sec:background}
In this section, for completeness, we first introduce the required background and the prior work, which include CFM (in \secref{sec:cfm}), the existing approaches to speeding up its inference by straightening the sampling paths (in \secref{sec:couplings}) and  optimal transport (in \secref{sec:ot}). We discuss the issues and the drawbacks of the existing methods and then explain our approach in \secref{sec:lbm}.

\subsection{Conditional Flow Matching}\label{sec:cfm}

Generative modelling requires estimating the unknown target data distribution $p(x)$ with some parametric model $p_\theta (x)$. A conventional choice for this estimator is $p_\theta(x) = \int p_\theta (x | z) p(z)\,dz$, where $p(z)$ is a given source noise distribution (typically the standard normal distribution) and $p_\theta (x | z)$ is a learned conditional noise-to-data distribution, or a generative distribution. Often the latter is considered to be deterministic, and in that case it can be written as $p_\theta(x | z) = \delta(x - g_\theta(z))$, where $g_\theta(z)$ is a deterministic mapping from noise to data, often referred to as the \emph{generator}, and $\delta(\cdot)$ is the Dirac delta distribution.

Recent ODE-based methods (such as denoising diffusion models~\citep{ho2020denoising, song2020denoising} or conditional flow matching~\citep{lipman2022flow, albergo2023stochastic, liu2023flow, tong2023conditional}) implicitly define the generator $g_\theta(z)$ through the following ODE
\begin{align}\label{eq:ode}
    \frac{d \phi(z, t)}{d t} &= v_\theta(\phi(z, t), t), \\
    \phi(z, 0) &= z,
\end{align}
where $v_\theta(y, t)$ is a vector field, parameterized with a neural network. $v_\theta(y, t)$ is trained in a way that $g_\theta(z) = \phi(z, 1)$ induces a valid approximation of the data distribution. One way to achieve this is by minimizing the following objective
\begin{align}
    {\cal L}_{\text{CFM}}(\theta) = \mathbb{E}_{x, z \sim p(x, z), t \sim U[0, 1], \varepsilon \sim {\cal N}(0, \sigma^2)} \left\| v_\theta(t x + (1 - t) z + \varepsilon, t) - (x - z) \right\|^2, \label{eq:cfm}
\end{align}

where $p(x, z)$ is some handcrafted coupling distribution with marginals equal to $p(x)$ and $p(z)$. A common choice of $p(x, z)$ is the so-called \emph{independent coupling} $p(x) p(z)$.

To sample from a model trained with the CFM objective, one needs to first draw a noise sample $z$ from the source distribution and then numerically solve the ODE to obtain $g_\theta(z)$. This requires discretizing the ODE and hence calling the neural network $v_\theta(y, t)$ multiple times, which can be costly and can slow down the generation. If the trajectories of the ODE had low curvature, fewer discretization steps would be needed in order to achieve the same accuracy.

In fact, despite the convex interpolation between a single noise and data sample pair in \eqref{eq:cfm}, being a straight line, the sampling trajectory of the ODE is actually far from straight. This can be seen even in a toy example, where the source and the target distributions are identical and are both standard normal Gaussians. Intuitively, the optimal vector field in such case should act as an identity. However, the sampling trajectories of a model trained in this setting first move towards the origin and then turn around and head back to the starting point (see \figref{fig:gaussian_to_gaussian_independent}). One might blame the convergence and numerical errors for such behavior. However, in this particular case there is a closed form analytical solution to the CFM optimization problem that is given by
\begin{align}
    v(y, t) = y \cdot \frac{2 t - 1}{\sigma^2 + t^2 + (1 - t)^2}.
\end{align}
This expression clearly explains the above observation (see Appendix~\ref{app:gtg} for the derivation). The reason why we get such a solution is the averaging in the loss function over all directions between the source and target samples (see \figref{fig:gaussian_to_gaussian_averaging}). This problem has already been noticed and pointed out in \citep{lee2023minimizing, esser2024scaling, lee2024improving}. Despite being just a toy example, the Gaussian to Gaussian case is actually important, as the conventional preprocessing in image generation involves normalizing images to ensure zero-mean and unit standard deviation.

\begin{figure}[t]
    \begin{subfigure}{0.3\textwidth}
        \includegraphics[width=\textwidth]{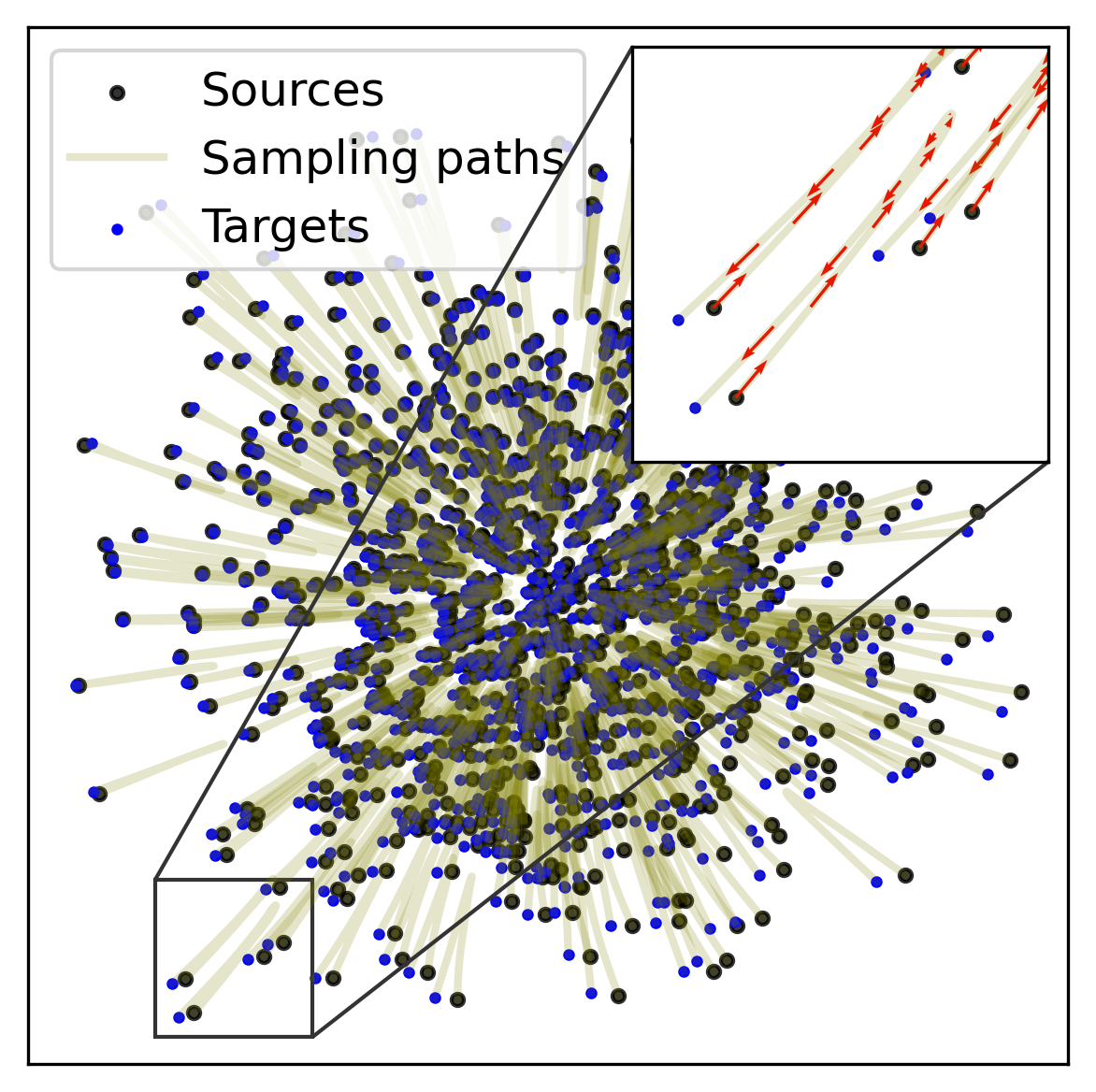}
        \caption{Independent coupling}
        \label{fig:gaussian_to_gaussian_independent}
    \end{subfigure}
    \hfill
    \begin{subfigure}{0.29\textwidth}
        \includegraphics[width=\textwidth, trim=20cm 5cm 20cm 5cm, clip]{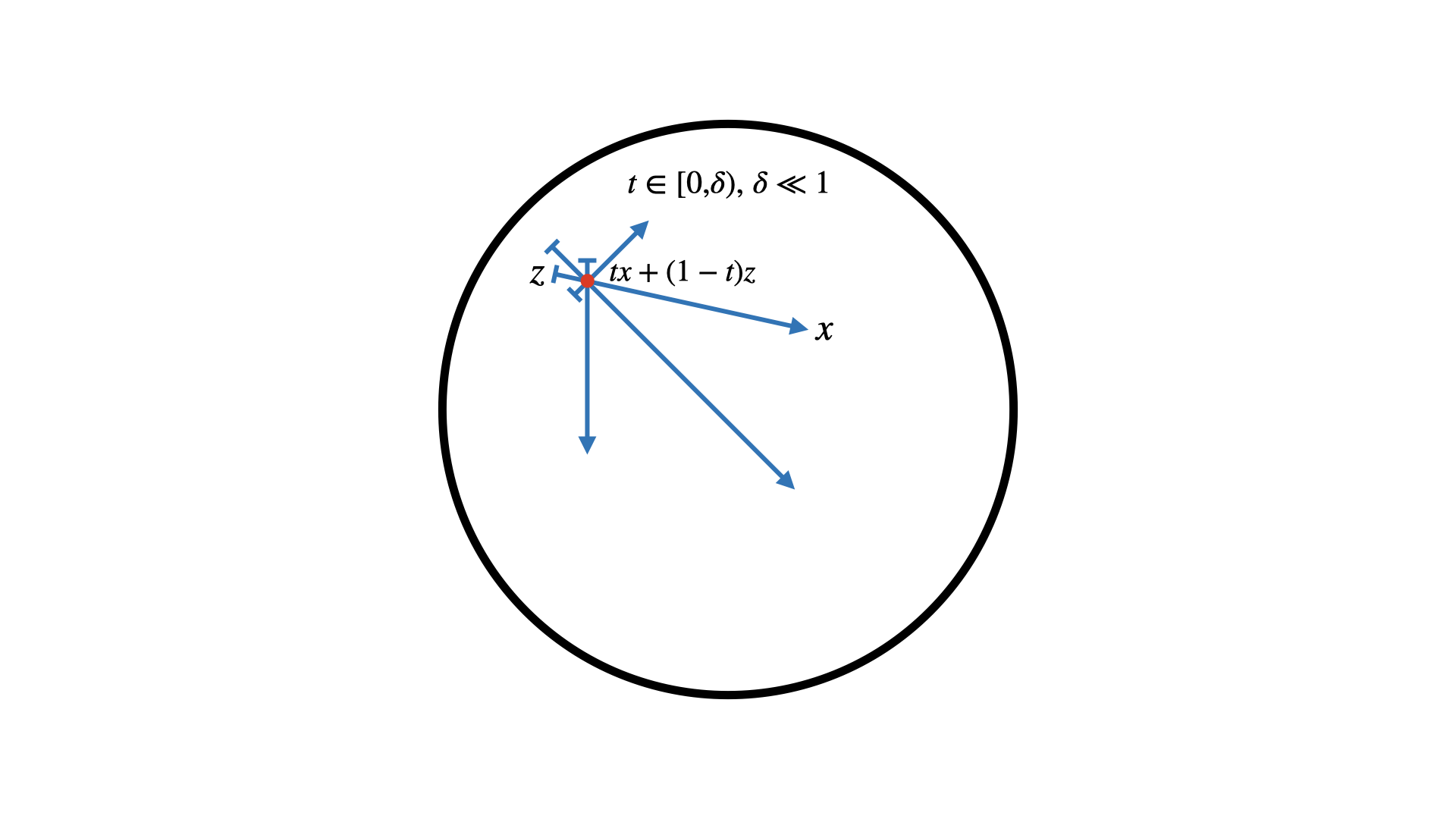}
        \caption{Vector field averaging}
        \label{fig:gaussian_to_gaussian_averaging}
    \end{subfigure}
    \hfill
    \begin{subfigure}{0.3\textwidth}
        \includegraphics[width=\textwidth]{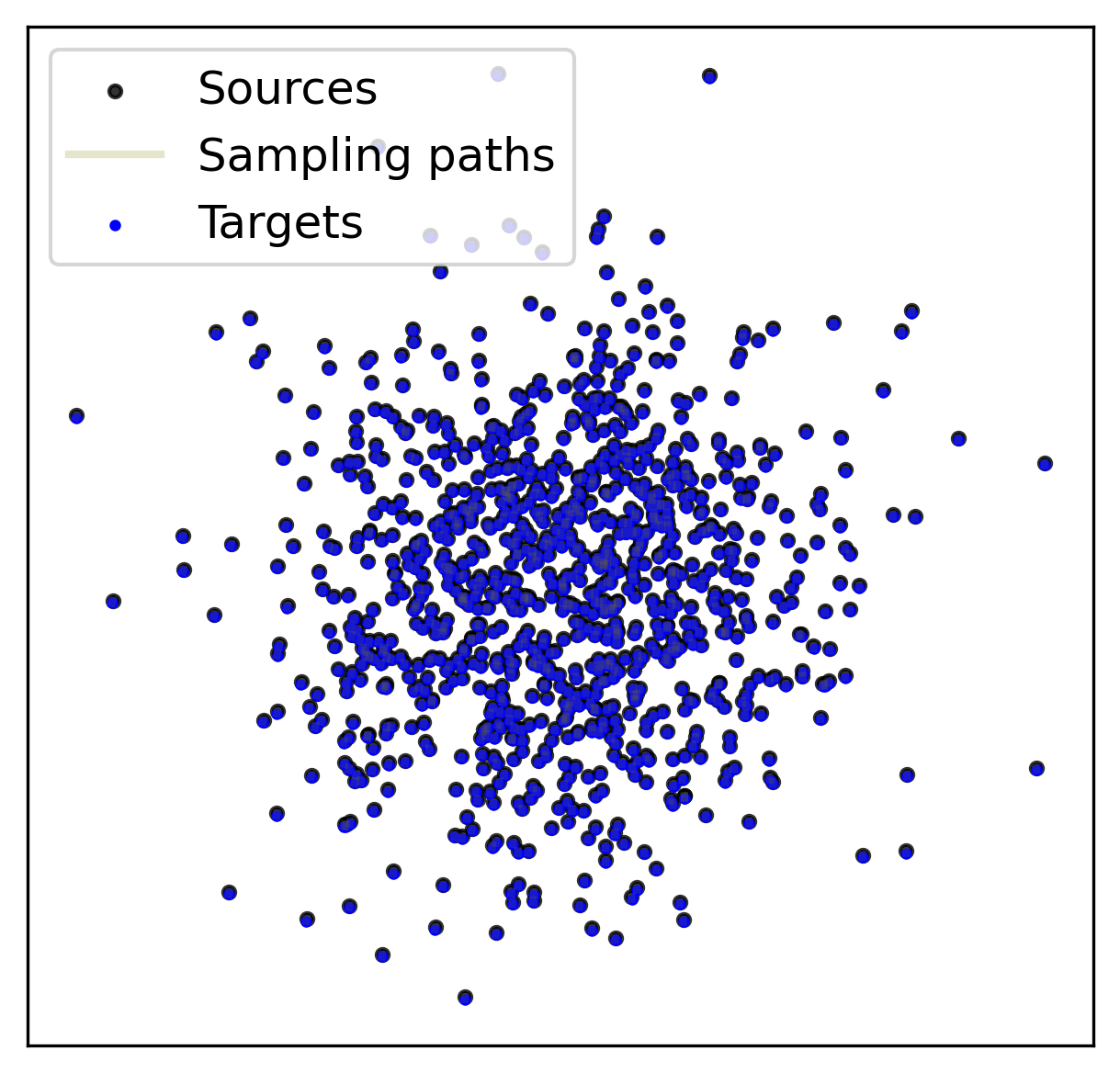}
        \caption{OT coupling}
        \label{fig:gaussian_to_gaussian_ot}
    \end{subfigure}
    \hfill
    \caption{2D Gaussian-to-Gaussian generation with CFM. This example illustrates the intuition behind the source of the curvature of the marginal sampling trajectories in the conditional flow matching framework and how optimal transport coupling could potentially solve this issue. (a) With independent coupling the sampling trajectories are highly curved. (b) This is due to the fact that in CFM the learned marginal vector field is an average over all source-to-target directions. As in the Gaussian case the distribution is symmetric around the origin, the learned vector field for small values of $t$ points towards the origin, since the target points are more likely to be sampled from the other side. (c) OT coupling resolves the curvature of the trajectories and results in the identity mapping via averaging only over the sources and targets that are close to each other.}
    \label{fig:gaussian_to_gaussian}
\end{figure}

\subsection{Faster Sampling and Data Couplings}
\label{sec:couplings}
Prior work has explored different approaches to speed up the integration of the ODE trajectories. One way to achieve faster generation is to reduce the number of integration steps and hence the number of function evaluations. In order to retain accuracy with fewer integration steps straighter sampling trajectories are desired. It has been shown that the straightness of the sampling paths is highly dependent on the number of criss-crosses of the data-to-noise interpolations that is induced by the coupling distribution $p(x, z)$~\citep{tong2023conditional, lee2023minimizing}. Hence, by only changing the coupling distribution, one can obtain straighter sampling trajectories of the learned marginal flow-field. In \citet{lee2023minimizing} the authors propose to model the coupling distribution as $p(x)p(z | x)$ and to also learn the encoder or the forward process $p(z | x)$. In \citet{bartosh2024neural} the forward process is also learned and since the coupling distribution is flexible, they can additionally impose explicit constraints on the straightness of the sampling paths. However, both approaches require training of auxiliary models and incorporate additional loss functions, which complicates the training overall. In this paper we focus instead on approaches that involve only changing the coupling distribution. It has been shown that the flow obtained after optimizing~\eqref{eq:cfm} induces a data coupling with a transport cost not larger than the cost of the initial coupling~\citep{liu2023flow}. More formally:
\begin{align}\label{eq:reflow}
    \mathbb{E}_{z \sim p(z)} [c(g_\theta(z), z)] \leq \mathbb{E}_{x, z \sim p(x, z)} [c(x, z)],
\end{align}
where $c(x, z)$ is an arbitrary convex function (\emph{e.g.}, $\|x - z\|^2$). Thus, using the induced coupling for the second training (the \textit{reflow} algorithm~\cite{liu2023flow}) tends to straighten the paths and lead to faster sampling. However, it comes with the burden of at least twice the training time. In~\citet{tong2023conditional} and~\citet{pooladian2023multisample} the authors noticed that one could instead use a coupling that is already optimal with respect to \eqref{eq:reflow} (see \figref{fig:gaussian_to_gaussian_ot}). Such a coupling is given by the so-called optimal transport plan (we introduce optimal transport in \secref{sec:ot}). Unfortunately, obtaining an exact optimal transport plan at the modern data scale is computationally infeasible. Therefore, in~\citet{tong2023conditional} and~\citet{pooladian2023multisample} the authors propose to approximate it via minibatch optimal transport. \citet{tong2023conditional} recover the soft permutation matrix with the Sinkhorn-Knopp alogirthm~\citep{sinkhorn} and sample from it as from the joint data-noise distribution, while \citet{pooladian2023multisample} calculate the hard permutation matrix with the Hungarian algorithm~\citep{kuhn1955hungarian} and reassign the data-noise pairs accordingly. However, the effectiveness of minibatch OT decreases with the growing sizes of the datasets. In contrast, our method, \methodName, is designed to improve minibatch OT by finding a better approximation to the global optimal transport plan via exchanging information across different minibatches. We describe our approach in \secref{sec:lbm}.

\subsection{Optimal Transport}
\label{sec:ot}
The optimal transport or Monge-Kantorovich problem refers to the search for an optimal coupling between two distributions $p(x)$ and $p(z)$, over $\mathcal{X}$ and $\mathcal{Z}$, with respect to a cost function $c(x, z)$ \citep{villani2009optimal}. Formally, a coupling between $p(x)$ and $p(z)$ is a distribution $p(x, z)$ over $\mathcal{X}\times\mathcal{Z}$  whose $x$ and $z$ marginals are exactly $p(x)$ and $p(z)$. In other words $\int_{\mathcal{Z}} p(x, z) dz = p(x)$ and $\int_{\mathcal{X}} p(x, z) dx = p(z)$. Among all such couplings, denoted $\Pi$, the solution to the following optimization problem
\begin{equation}
\inf_{p(x,z) \in \Pi} \E_{x,z \sim p(x, z)}[c(x, z)],
\label{eq:formalOT}
\end{equation}
is referred to as the \emph{optimal coupling}. Remarkably, the optimal coupling can be shown to be deterministic for distributions over $\R^d$ under minimal assumptions on $c$.  
Deterministic couplings $p(x, z)$ are those that can be expressed as the joint distribution of $(x, T(x))$ where $x \sim p(x)$ and $T: \mathcal{X} \rightarrow \mathcal{Z}$. The map $T$ \emph{transports} $x \sim p(x)$ to $T(x) \sim p(z)$. In our setting, solving \eqref{eq:formalOT} thus boils down to finding the optimal transport map $T$. In fact, further simplifications can be made since, in this paper, $p(x)$ will correspond to a uniform distribution over $n$ data samples, and $p(z)$ will be a uniform distribution over $n$ noise samples. Consequently, the problem considered in this paper is the optimal transport problem between two uniform distributions with finite support of equal size. 

Several algorithms exist for tackling this optimization task, see \citet{schrieber2016dotmark} for an extensive list. The problem can be cast as either a linear program or a graph-matching problem, and the most notable scheme for solving it is the Hungarian algorithm \citep{kuhn1955hungarian} which can be understood as either a primal-dual method for solving the linear program or an augmenting path approach to finding a minimum cost matching (see 3.6 in \citet{peyre2019computational}). The complexity of solving \eqref{eq:formalOT} between uniform distributions over $n$ points is $\mathcal{O}(n^3 \log(n))$. This complexity quickly becomes prohibitive for large-scale problems, which motivated the search for tractable approximations. The work of \citet{cuturi2013sinkhorn}, for instance, showed that with additional regularization of \ref{eq:formalOT}, the resulting surrogate optimization problem can be solved with a reduced $\mathcal{O}(n^2)$ complexity but remains out of reach for modern datasets. 

Alternative approximations, more relevant to our work, are randomized block-coordinate \citep{xie2024randomized} and mini-batch approaches (mOT) \citep{fatras2021minibatch} which operate on subproblems of size $m < n$.  The mini-batch approach, which is used in \citet{pooladian2023multisample}, consists of iteratively sampling a subset of $m$ data points and $m$ noise points independently at each iteration and determining the optimal coupling with a complexity of $\mathcal{O}(m^3)$. This scheme \emph{does not converge to the optimal coupling} but to a sub-optimal, non-deterministic, averaged coupling whose bias has been analyzed in \citet{sommerfeld2019optimal} and greatly refined in \citet{fatras2021minibatch}. On the other hand, the work of \citet{xie2024randomized} proposes a randomized coordinate selection scheme to solve the linear programming formulation of \ref{eq:formalOT} using a sequence of linear programs over a reduced number of variables. Convergence to optimality can be established as long as the subproblems are solved over $3n$ variables, which again remains prohibitively expensive. 

Our work can be seen as an intermediate between those two approaches. At each iteration, we sample $m$ data points and their \emph{corresponding} noise samples determined at the previous iteration, instead of the independently sampled noise points of \citet{fatras2021minibatch}. The complexity of our scheme does not exceed $\mathcal{O}(m^3)$ and is a block coordinate update scheme like \citet{xie2024randomized}. Our method converges to a stationary transport plan, instead of an averaged one in mOT, albeit to a sub-optimal one. We describe our method in more detail in the next section.

\section{Method}

In this section, we first introduce our approach to finding better data-to-noise couplings (\secref{sec:lbm}). Then, we propose a simple modification to the method to prevent overfitting to the fixed source (\secref{sec:multiple_caches}). Finally, we provide some implementation details (\secref{sec:implementation}).

\subsection{Looking Out Of the Minibatch}
\label{sec:lbm}
Suppose that we are given a set of training data points $\{x_i\}_{i \in \mathbb{N}_n}$ and noise samples $\{z_i\}_{i \in \mathbb{N}_n}$, where $\mathbb{N}_n = \{1, \dots, n\}$ and $x_i, z_i \in \mathbb{R}^d$. We aim to find a bijection $\tau^*: \mathbb{N}_n \rightarrow \mathbb{N}_n$ that represents an optimal coupling between the data and the noise sets. That is, we seek a permutation of $\mathbb{N}_n$ that satisfies
\begin{align}\label{eq:assignment}
    \tau^* = \arg\min_{\tau \in S_n} \sum_{i \in \mathbb{N}_n} c(x_i, z_{\tau(i)}),
\end{align}
where $c(x, z)$ is the transport cost between $x$ and $z$, and $S_n$ is the symmetric group (of degree $n$) of $\mathbb{N}_n$. In this paper, we consider the cost $c(x, z) = \|x - z\|_2$. Let $P_\tau$ be the permutation matrix corresponding to the mapping $\tau$ and let $X$ and $Z \in \mathbb{R}^{n \times d}$ be the data and noise matrices formed by stacking the corresponding samples row-wise. Then we can rewrite \eqref{eq:assignment} as:
\begin{align}
\label{eq:assignment_perm}
    P_{\tau^*} = \arg\min_{\tau \in S_n} \left\|X - P_\tau Z\right\|_{2,1},
\tag{OT}
\end{align}
where $\|\cdot\|_{2,1}$ denotes the sum of the Euclidean norms of the rows of a matrix.
Due to the finiteness of the feasible set, it is clear that a solution to the above problem always exists. Ideally, we would like to use $\tau^*$ as a data coupling to train CFM. That is, we would define the coupling distribution as
\begin{align}
    p_\tau(x, z) = \frac{1}{n}\sum_{i = 1}^n \delta_{x_i, z_{\tau(i)}}
\end{align}
and use it to optimize \eqref{eq:assignment}. However, algorithms to find $\tau^*$ usually require $O(n^3)$ running time (\emph{e.g.}, Hungarian algorithm~\citep{kuhn1955hungarian}), which makes their use infeasible at the deep learning data scale.  Existing work, such as \citep{tong2023conditional, pooladian2023multisample}, approximates $\tau^*$ with a minibatch version of \eqref{eq:assignment}. At each iteration, a new minibatch of noise is sampled and a local optimal transport is solved to assign noise samples to data within that minibatch. These assignments are used to perform the training step. Despite the simplicity of this method, it suffers from a limited scalability to larger datasets and data dimensionalities, as local assignments might be globally suboptimal. To address this issue, we propose an iterative procedure to better approximate the globally optimal assignment. First, $\{z_i\}_{i \in \mathbb{N}_n}$ are sampled and assigned to the corresponding $\{x_i\}_{i \in \mathbb{N}_n}$. At the beginning, $\tau_0$ is set to the trivial identity permutation, \ie, $\tau_0(i) = i, \forall i \in \mathbb{N}_n$. At the $k$-th training iteration, a minibatch of $\{x_{n_j}\}_{j = 1}^m$ and the correspoding $\{z_{\tau_{k-1}(n_j)}\}_{j = 1}^m$ is sampled from $p_{\tau_{k-1}}(x, z)$. $\tau_{k-1}$ is then locally updated to ensure the optimality of \eqref{eq:assignment} restricted to the minibatch. That is, we find another permutation $\omega_k$ that acts locally on the current minibatch (\ie, $\omega_k(i) = i, \,\text{if}\; i \notin \{\tau_{k-1}(n_j)\}_{j = 1}^m$) such that the following objective is minimized
\begin{align}\label{eq:assignment_local}
    \omega_k = \arg\min_{\omega \in S_m} \sum_{i = 1}^m c\left(x_{n_i}, z_{\omega(\tau_{k-1}(n_i))}\right).
\end{align}
Hence, the update for $\tau_{k - 1}$ takes the form:
\begin{align}
    \tau_k = \omega_k \circ \tau_{k - 1}.
\end{align}
Finally, the training step is performed using the updated assignments from \eqref{eq:assignment_local}, which are then saved for later iterations. Formally the method is described in Algorithm~\ref{alg:loom} and an illustration is provided in \figref{fig:loom}.

Notice that in contrast to the prior work, the solution to the local assignment problem is not thrown away after each step, but is instead used to update the global assignment and affects the future minibatches. This makes our approach strictly better than the prior work in terms of global optimality and approximation error. 

The proposed Algorithm~\ref{alg:loom} is guaranteed to converge to a stable solution.  The following theorem, whose proof can be found in Appendix~\ref{app:analysis}, characterizes precisely the convergence behavior of \methodName. The analysis is based on casting our scheme as a randomized cycle elimination method which converges to a stationary solution.
\begin{restatable}[Finite convergence]{theorem}{convergencethm}
     \methodName generates a sequence of assignments $\tau_k$ with non-increasing costs. With probability $1$ over the random batch selection, the iterates converge in a finite number of steps to a final assignment $\tau_{\text{final}}$ whose associated matching $M_{\text{final}}$ has no negative alternating cycles of length less than $m$. 
\end{restatable}
%This result, shows that the algorithm converges to a stationary solution. % Unlike the global optimum characterized in \ref{prop:globopt}, negative cycles of large lengths can still exist. The stationary solution is thus not guaranteed to be globally optimal.
Like all tractable approximation schemes, our algorithm is not guaranteed to recover the optimal coupling. However, unlike the minibatch OT schemes, our method outputs a deterministic coupling, and more importantly, we show in \secref{sec:image_generation} that it induces a significantly better straightening of sampling paths.

\begin{figure}[t]
\noindent
\begin{minipage}[t]{.55\linewidth}
    \vspace{0pt}  % Ensures the vertical alignment is top
    \centering
    \includegraphics[width=\linewidth, trim=5cm 6cm 2cm 3.2cm, clip]{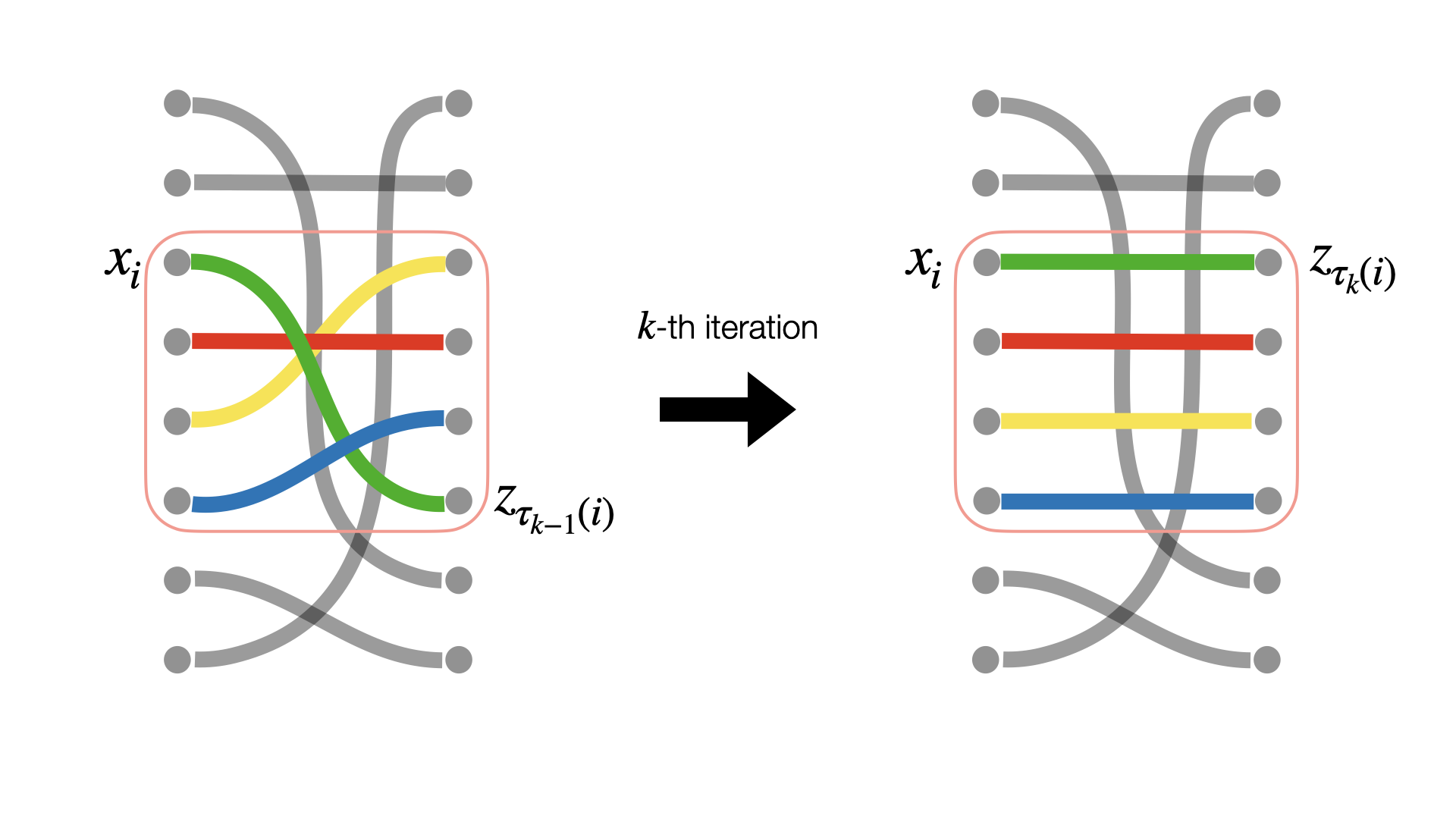}
    \captionof{figure}{At each iteration a random minibatch is sampled according to the current data-to-noise assignments. Prior to taking the gradient step the current assignment is locally updated to ensure optimality w.r.t. the current minibatch. This process resembles weaving on a loom, hence the name of the method \methodName, or \methodNameExtended CFM.}
    \label{fig:loom}
\end{minipage}%
\hfill % Adds space between the two minipages
\begin{minipage}[t]{.4\linewidth}
    \vspace{0pt}  % Ensures the vertical alignment is top
    \begin{algorithm}[H]
        \caption{\methodName}\label{alg:loom}
        \begin{algorithmic}[1]
            \STATE \textbf{Input:} \\ Set of data points $\{x_i\}_{i \in \mathbb{N}_n}$, \\ Set of noise samples $\{z_i\}_{i \in \mathbb{N}_n}$, \\ Initial assignment $\tau_0 = \text{Id}$;
            \FOR{$k$ in range(1, $T$)}
            \STATE Sample minibatch \\$\{x_{n_j}$, $z_{\tau_{k - 1}(n_j)}\}_{j = 1}^m$ $\sim$ $p_{\tau_{k-1}}(x, z)$;
            \STATE Calculate $\omega_k$ as in \eqref{eq:assignment_local};
            \STATE Update $\tau_k \leftarrow \omega_k \circ \tau_{k - 1}$;
            \STATE Take a gradient descent update w.r.t. \eqref{eq:assignment} on \\$\{x_{n_j}$, $z_{\tau_k(n_j))}\}_{j=1}^m$;
            \ENDFOR
            \STATE \textbf{Return:} $\tau_T$
        \end{algorithmic}
    \end{algorithm}
\end{minipage}
\end{figure}

\subsection{Multiple Noise Caches}
\label{sec:multiple_caches}
When the size of the dataset is not large enough, using a fixed set of noise samples as the source distribution may lead to overfitting. As a result, using new noise instances as starting points for the numerical integration of the ODE in \eqref{eq:ode} may lead to samples of poor quality. To overcome this issue, we propose to store more than one assigned noise sample per data point in the dataset. At each training iteration a minibatch of data-noise pairs is obtained by first sampling data points and then randomly picking one of the assigned noises. This corresponds to artificially enlarging the dataset by duplicating the data points and does not change the underlying data distribution. Later in the paper, we refer to the number of the assigned noise samples as the number of \textit{noise caches} and show empirically that this technique helps to prevent overfitting and allows for using new noise instances as source points at inference.

\subsection{Speed and Memory Analysis}
\label{sec:implementation}
% For finding the locally optimal assignments at each training iteration, we utilize the Hungarian algorithm~\citep{kuhn1955hungarian}. This makes \methodName comparable to OT-CFM~\citep{tong2023conditional} or BatchOT~\citep{pooladian2023multisample} in terms of time complexity, as it solves the same minibatch matching problem. However, saving and loading the current assignments from the disk introduces a small i/o overhead. This leads to slightly longer processing per minibatch. Nevertheless, this overhead is compensated by the accelerated convergence of \methodName. In our experiments on ImageNet-32 and -64, LOOM-CFM was trained for only 200 and 100 epochs, respectively, compared to BatchOT's 350 and 575 epochs.
% Additionally, to reduce the disk usage, instead of saving the assigned noise instances, we store the \textbf{seed values} used in the random numbers generator to generate the corresponding samples. This incurs a negligible disk usage overhead compared to that required for storing and loading modern datasets sizes. For example, if a dataset contains 1M images, the size of a single noise cache for this dataset would be under 4MB. For implementation details, see Appendix~\ref{app:impl}.
We use the Hungarian algorithm~\citep{kuhn1955hungarian} to find locally optimal assignments at each training step, making \methodName comparable in time complexity to OT-CFM~\citep{tong2023conditional} and BatchOT~\citep{pooladian2023multisample}, as it solves the same minibatch matching problem. While saving and loading assignments adds minor I/O overhead, it is offset by \methodName's faster convergence. For instance, on ImageNet-32 and -64, \methodName required only 200 and 100 epochs, compared to BatchOT’s 350 and 575. To minimize disk usage, we store the random numbers generator's \textbf{seed values} instead of the assigned noise instances. This incurs a negligible disk usage overhead compared to that required for storing and loading modern datasets sizes. For example, a 1M-image dataset requires only under 4MB for the noise cache. For implementation details, see Appendix~\ref{app:impl}.

\section{Experiments}

In this section, we present quantitative experiments to demonstrate the effectiveness of our method on real-world data. First, in Sec.\ref{sec:ablations}, we perform ablation studies to analyze the contribution of different components of \methodName. In Sec.\ref{sec:image_generation}, we compare our approach to prior work, showing that \methodName, as designed, generates higher-quality samples with fewer integration steps. Then we illustrate how \methodName enhances the initialization of the Reflow algorithm\citep{liu2023flow}, removing the need for multiple Reflow iterations. Lastly, we show that our method is compatible with training in the latent space of a pre-trained autoencoder, enabling higher-resolution synthesis.

In all experiments, the main reported metric is FID~\citep{fid} which measures the distance between the distributions of some pre-trained deep features of real and generated data. It has been proven to correlate well with human perception and established as a conventional metric for image quality. As commonly done in the literature, we report the FID on the validation set with respect to 50K generated samples. The other measure that is reported is NFE (Number of Function Evaluations) that refers to the number of integration steps used to produce the corresponding result. The combination of the FID and the NFE acts as a proxy measure for the curvature of the trajectories. Low FID values along with low NFE speak for the straightness of the sampling trajectories, as the model is able to generate high-quality images using small number of steps. For qualitative results, see Appendix~\ref{app:qualitative_results}.

\subsection{Ablations}
\label{sec:ablations}
In this section, we test different versions of our algorithm to understand the impact of each component on the final performance. All ablations are conducted for unconditional generation on CIFAR10~\citep{cifar10}, a dataset of $32 \times 32$ resolution images from 10 classes containing 50k training and 10k validation images. The metrics are reported on the validation set. 

\noindent\textbf{Training with fixed source:} We start by training a naive version of CFM, by fixing the initial random assignment between the collection of noise samples and the images. Interestingly, while the results are worse than those of \methodName (see \figref{fig:abl_abl} and Appendix~\ref{app:detailed_results}, ``Fixed source, w/o reassignments''), the FID at 4 and 8 NFE are better than those of OT-CFM~\citep{tong2023conditional}.

\noindent\textbf{Training after convergence:} A logical variant of our approach would be to wait till the algorithm converges to the stable matching and only then start training the vector field. However, this version results in slightly worse results (see \figref{fig:abl_abl} and Appendix~\ref{app:detailed_results}, ``Train After Convergence''). This might be related to overfitting to the fixed collection of noise samples and the corresponding matching. Indeed, since the algorithm starts with a random matching, training the network from the very beginning, while the matching keeps improving, can be viewed as a soft transition between the independent coupling and the coupling induced by our method.

\noindent\textbf{Training without saving reassignments:} Next, we test the impact of updating the global matching with the local reassignments. Instead of using a fixed set of noises and exchanging the information between different minibatches, for a given minibatch of images, we always sample a new minibatch of noises from the standard Gaussian distribution and assign them to the images by calculating the optimal matching. This variant replicates the method of MFM $^\text{w}/$ BatchOT~\citep{pooladian2023multisample}. While the quality of generated images with this method is comparable to the quality of \methodName at large NFE, with few sampling steps \methodName is significantly better (see \figref{fig:abl_abl} and Appendix~\ref{app:detailed_results}, ``W/o saving local reassignments''). Moreover, we also explicitly demonstrate that the curvature of the trajectories, measured as $1 - \bar v_\theta(\phi(z, t), t)^\top \bar v_\theta(\phi(z, t + \Delta t), t + \Delta t)$, is smaller with \methodName (see \figref{fig:abl_curv}). Here $\bar v_\theta$ is the normalized vector field and $t$ and $t + \Delta t$ are two consecutive integration timestamps.

\noindent\textbf{Batch size:} Further, we test \methodName with different batch sizes. As shown in \figref{fig:abl_bs}, with large NFE different batch sizes lead to similar performance. However, with small NFE larger batch size significantly overperforms the smaller ones. This indicates that the method trained with a larger batch size tends to yield straighter sampling trajectories. Nevertheless, we would like to point out that \methodName trained with batch size 32 produces on par results for 8 and 12 NFE with OT-CFM~\citep{tong2023conditional} that was trained with batch size 128 and outperforms it with larger batch sizes. This demonstrates the effectiveness of our caching scheme.

\noindent\textbf{Number of noise caches:} Finally, we ablate the number of noise caches used for training. \figref{fig:abl_nc} shows that larger number of noise caches improves the performance. We found empirically, that 4 caches are enough for a dataset of size comparable to CIFAR10. We observed that the improvement diminishes with more caches, as the method becomes slower to converge and discovers worse couplings as the number of noise caches grows (see Appendix~\ref{app:detailed_results}). However, if the dataset is large enough (e.g. ImageNet~\citep{imagenet}), even a single cache may be sufficient.

\begin{figure}[t]
\begin{minipage}{.49\linewidth}
    \centering
    \includegraphics[width=0.95\linewidth]{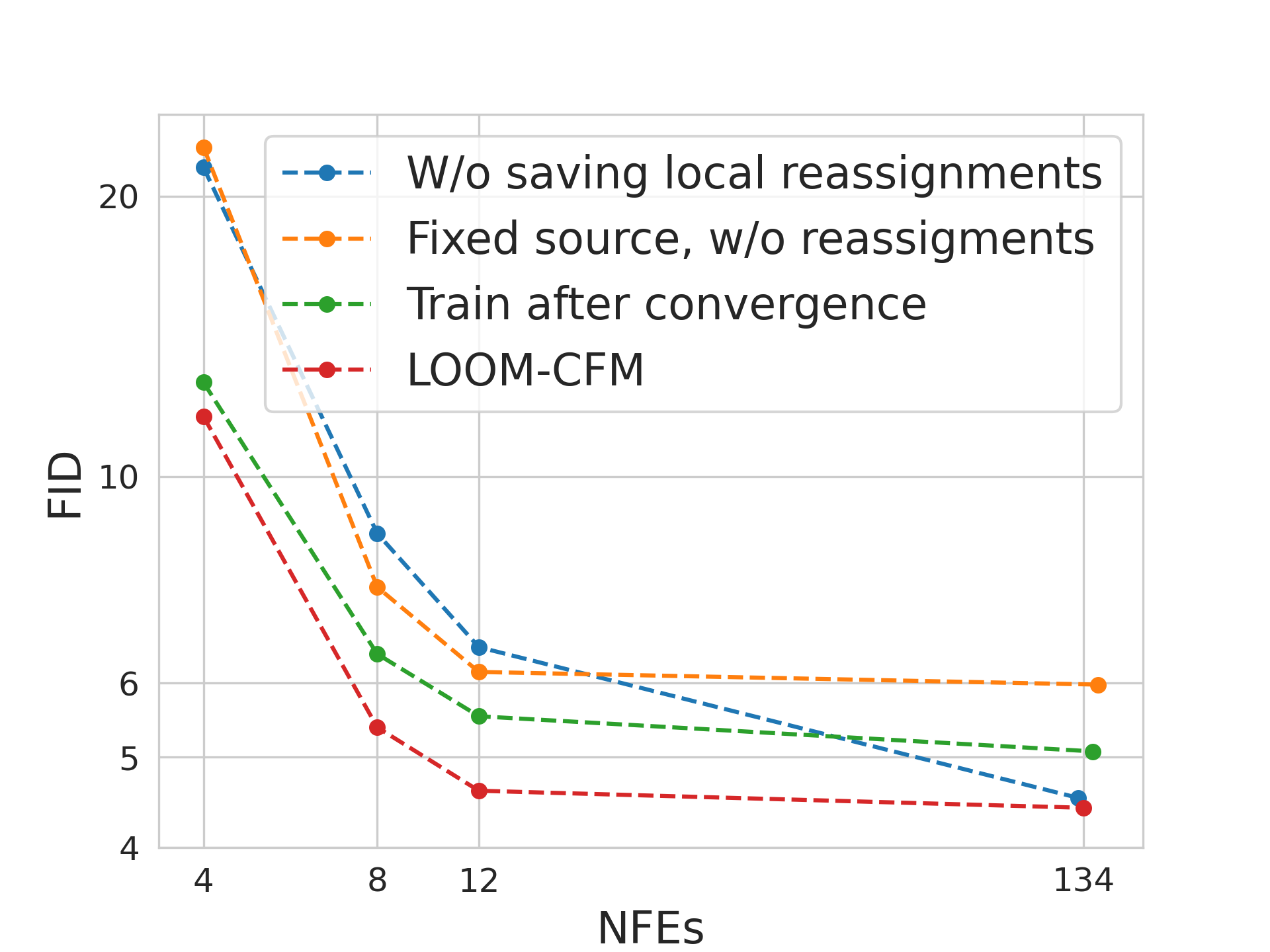}
    \captionof{figure}{Ablations on different design choices on CIFAR10.}
    \label{fig:abl_abl}
\end{minipage}%
\hfill%
\begin{minipage}{.49\linewidth}
    \centering
    \includegraphics[width=0.95\linewidth]{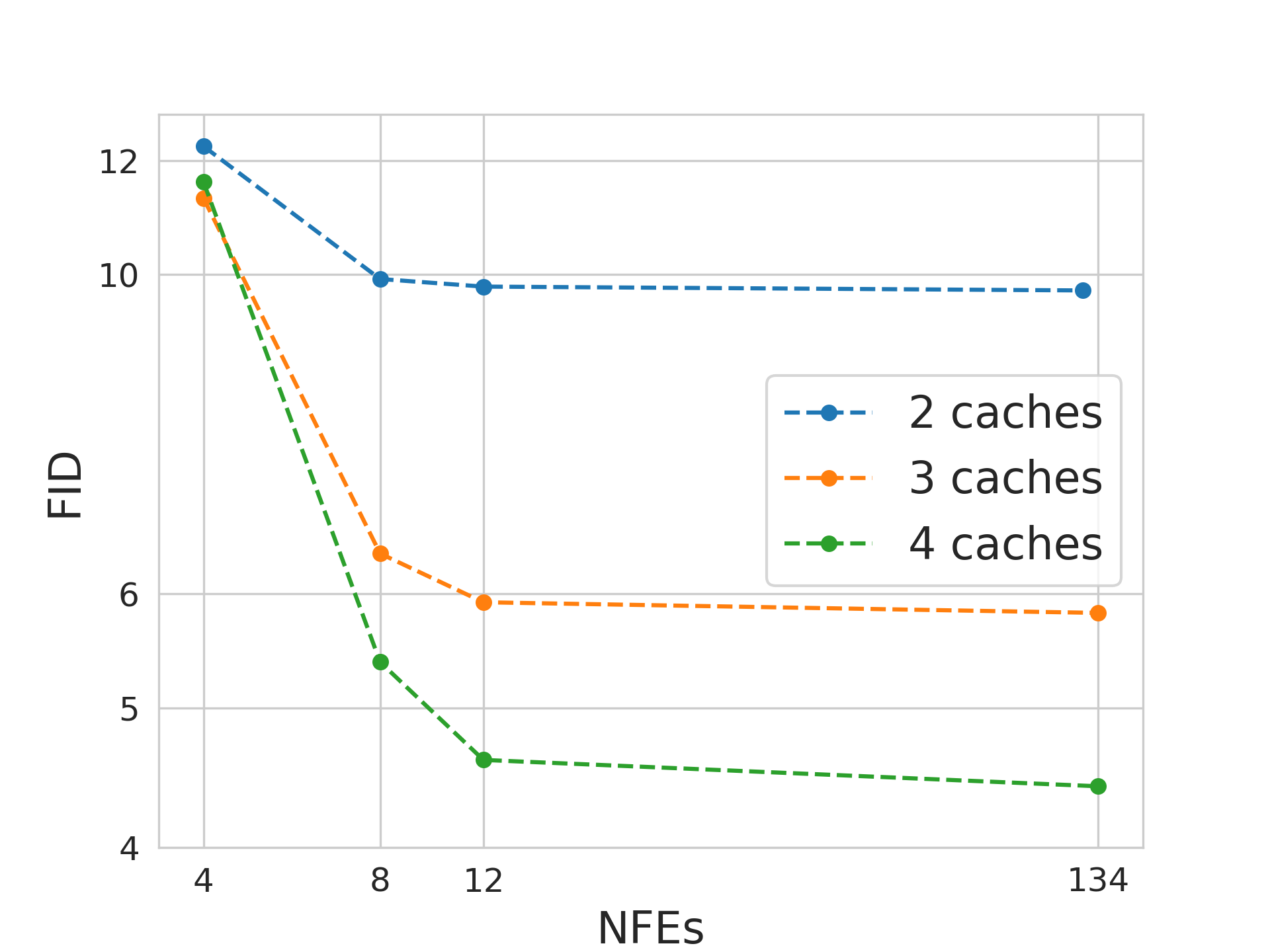}
    \captionof{figure}{Ablation on the number of caches used on CIFAR10.}
    \label{fig:abl_nc}
\end{minipage}%
\\
% \end{figure}%
% \begin{figure}[t]
\begin{minipage}{.49\linewidth}
    \centering
    \includegraphics[width=0.95\linewidth]{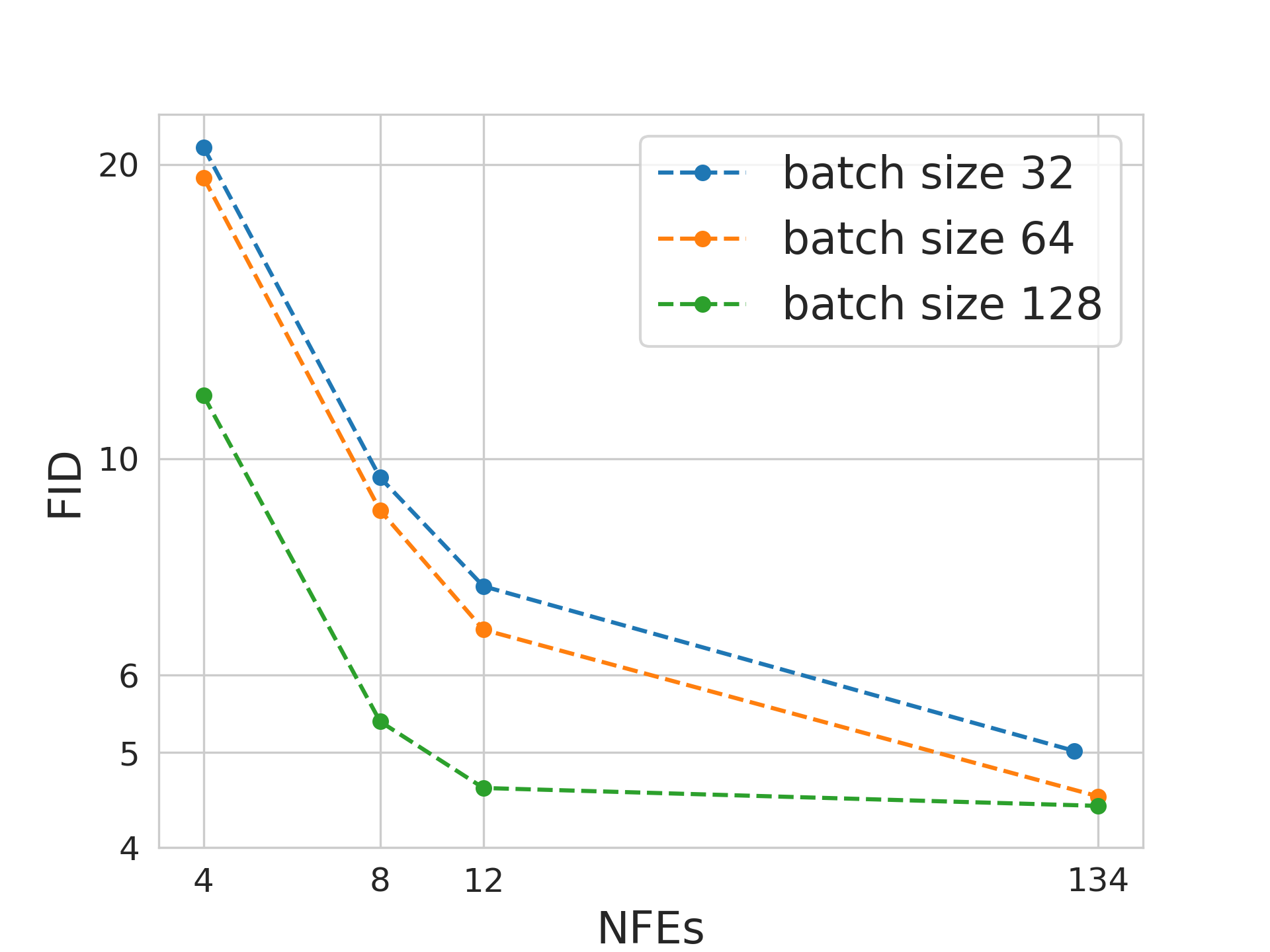}
    \captionof{figure}{Ablation on the size of the minibatch used on CIFAR10.}
    \label{fig:abl_bs}
\end{minipage}%
\hfill%
\begin{minipage}{.49\linewidth}
    \centering
    \includegraphics[width=0.95\linewidth]{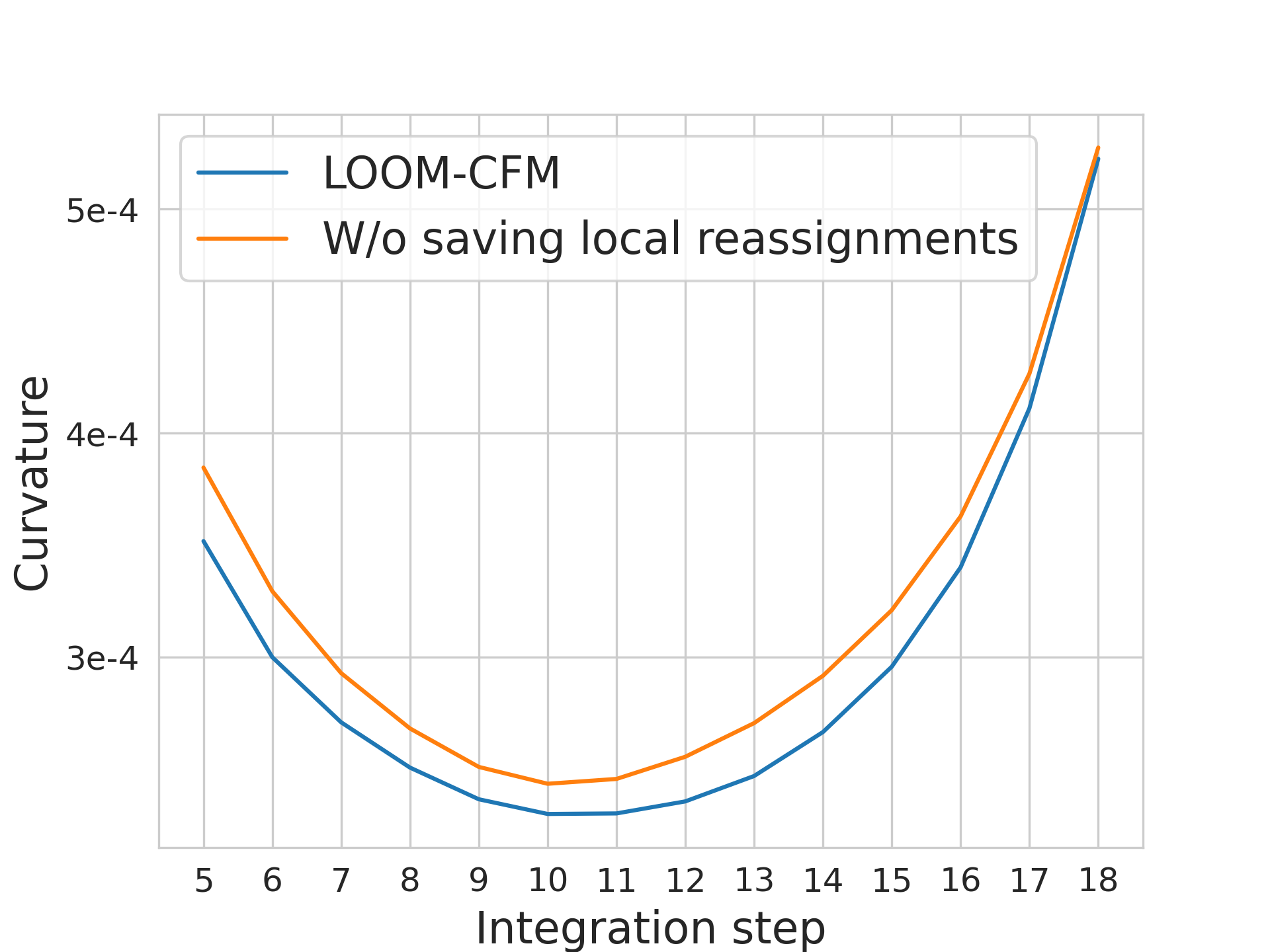}
    \captionof{figure}{Curvature of the sampling trajectories for the models trained on CIFAR10. Average of 1000 trajectories is reported.}
    \label{fig:abl_curv}
\end{minipage}
\end{figure}%

\subsection{Results}\label{sec:image_generation}
\noindent\textbf{Unconditional Image Generation.} We train \methodName for unconditional generation on CIFAR10~\citep{cifar10} and ImageNet32/64~\citep{imagenet}, with the results presented in Tables~\ref{tab:cifar_results} and~\ref{tab:image_net_results}, respectively. \methodName consistently achieves better few-step FID scores compared to minibatch OT methods, such as OT-CFM~\citep{tong2023conditional} and MFM $^\text{w}/$ BatchOT~\citep{pooladian2023multisample}. Additionally, \methodName outperforms \citet{lee2023minimizing} and delivers on par quality with NFDM-OT~\citep{bartosh2024neural}, surpassing it with 12 NFE while slightly underperforming at lower NFE. Unlike these methods, which require training additional components, \methodName optimizes the original CFM objective with a modified coupling distribution.

% \begin{minipage}[t]{.5\linewidth}
\begin{table}[t]
    \centering
    \footnotesize
    \caption{Unconditional generation results on CIFAR10. ($^*$) indicates that the reported numbers are calculated by us using the official code, as they were not reported in the corresponding papers.}
    \label{tab:cifar_results}
    \begin{tabular}{lccr}
        \toprule
        Method & Solver & \multicolumn{2}{@{}c@{}}{CIFAR10} \\
        & & NFE & FID$\downarrow$ \\
        \hline
        \multicolumn{4}{@{}l}{\textit{ODE/SDE-based methods}} \\
        DDPM~\citep{ho2020denoising} & - & 1000 & 3.17 \\
        DDPM~\citep{ho2020denoising} & dopri5 & 274 & 7.48 \\
        St. Interpolatns~\citep{albergo2023stochastic} & - & - & 10.27 \\
        FM $^\text{w}/$ OT~\citep{lipman2022flow} & dopri5 & 142 & 6.35 \\
        I-CFM~\citep{tong2023conditional} & euler & 100 & 4.46 \\
        I-CFM~\citep{tong2023conditional} & euler & 1000 & 3.64 \\
        I-CFM~\citep{tong2023conditional} & dopri5 & 146 & 3.66 \\
        \hline
        \multicolumn{4}{@{}l}{\textit{Improved sampling}} \\
        DDIM~\citep{song2021denoising} & - & 10 & 13.36 \\
        DPM Solver-2~\citep{lu2022dpm} & - & 12 & 5.28 \\
        DPM Solver-3~\citep{lu2022dpm} & - & 24 & 2.75 \\
        RK2-BES~\citep{shaul2024bespoke} & - & 10 & 2.73 \\
        RK2-BES~\citep{shaul2024bespoke} & - & 20 & 2.59 \\
        \hline
        \multicolumn{4}{@{}l}{\textit{Rectified Flows}} \\
        1-Rectified Flow~\citep{liu2023flow} & euler & 1 & 378 \\ 
        \qquad + distill & euler & 1 & 6.18 \\ 
        2-Rectified Flow~\citep{liu2023flow} & euler & 1 & 12.21 \\
        \qquad + distill & euler & 1 & 4.85 \\ 
        3-Rectified Flow~\citep{liu2023flow} & euler & 1 & 8.15 \\
        \qquad + distill & euler & 1 & 5.21 \\ 
        \hline
        \multicolumn{4}{@{}l}{\textit{Improved training for straighter trajectories}} \\
        % OT-CFM~\citep{tong2023conditional}$^*$ & euler & 4 & 30.40 \\
        OT-CFM~\citep{tong2023conditional}$^*$ & midpoint & 4 & 15.95 \\
        OT-CFM~\citep{tong2023conditional}$^*$ & midpoint & 8 & 9.73 \\
        OT-CFM~\citep{tong2023conditional}$^*$ & midpoint & 12 & 7.77 \\
        % OT-CFM~\citep{tong2023conditional}$^*$ & euler & 8 & 16.28 \\
        % OT-CFM~\citep{tong2023conditional}$^*$ & euler & 12 & 12.31 \\
        OT-CFM~\citep{tong2023conditional} & euler & 100 & 4.44 \\
        OT-CFM~\citep{tong2023conditional} & dopri5 & 134 & 3.57 \\
        Minimizing Trajectory Curvature~\citep{lee2023minimizing} & heun & 5 & 18.74 \\
        NFDM-OT~\citep{bartosh2024neural} & euler & 2 & 12.44 \\
        NFDM-OT~\citep{bartosh2024neural} & euler & 4 & 7.76 \\
        NFDM-OT~\citep{bartosh2024neural} & euler & 12 & 5.20 \\
        \hline
        \multicolumn{4}{@{}l}{\textit{Our results}} \\
        % \methodName (\textit{ours}) 4 caches & midpoint & 4 & 11.06 \\ % & 4 & 12.17 \\
        % \methodName (\textit{ours}) 4 caches & midpoint & 8 & 4.95 \\ % 8 & 6.16 \\
        % \methodName (\textit{ours}) 4 caches & midpoint & 12 & 4.29 \\ % 12 & 5.42 \\
        % \methodName (\textit{ours}) 4 caches & dopri5 & 134 & 4.17 \\ % 135 & 5.23 \\
        % \methodName (\textit{ours}) 4 caches + \textit{reflow} & euler & 1 & 7.30 \\ % 1 & 9.05 \\
        % \methodName (\textit{ours}) 4 caches + \textit{reflow} & midpoint & 2 & 4.19 \\ % 2 & 5.23 \\
        \methodName (\textit{ours}) 4 caches & midpoint & 4 & 11.60 \\ % & 4 & 12.17 \\
        \methodName (\textit{ours}) 4 caches & midpoint & 8 & 5.38 \\ % 8 & 6.16 \\
        \methodName (\textit{ours}) 4 caches & midpoint & 12 & 4.60 \\ % 12 & 5.42 \\
        \methodName (\textit{ours}) 4 caches & dopri5 & 134 & 4.41 \\ % 135 & 5.23 \\
        \methodName (\textit{ours}) 4 caches + \textit{reflow} & euler & 1 & 7.63 \\ % 1 & 9.05 \\
        \methodName (\textit{ours}) 4 caches + \textit{reflow} & midpoint & 2 & 4.49 \\ % 2 & 5.23 \\
        \bottomrule
    \end{tabular}
\end{table}
% \end{minipage}

% \begin{table}[t]
%     \centering
%     \begin{tabular}{lccr}
%         \toprule
%         Method & Solver & \multicolumn{2}{@{}c@{}}{CIFAR10} \\
%         & & NFE & FID$\downarrow$ \\
%         \hline
%         4 caches & midpoint & 4 & 11.06 \\ % & 4 & 12.17 \\
%         4 caches & midpoint & 8 & 4.95 \\ % 8 & 6.16 \\
%         4 caches & midpoint & 12 & 4.29 \\ % 12 & 5.42 \\
%         4 caches & dopri5 & 134 & 4.17 \\ % 135 & 5.23 \\
%         4 caches + \textit{reflow} & euler & 1 & 7.30 \\ % 1 & 9.05 \\
%         4 caches + \textit{reflow} & midpoint & 2 & 4.19 \\ % 2 & 5.23 \\
%         \bottomrule
%     \end{tabular}
%     \caption{Ablations on CIFAR10.}
%     \label{tab:cifar_reflow}
% \end{table}

\begin{table}[t]
    \centering
    \footnotesize
    \caption{Unconditional generation results on ImageNet-32 and ImageNet-64}
    \label{tab:image_net_results}
    \begin{tabular}{lccrcr}
        \toprule
        Method & Solver & \multicolumn{2}{@{}c@{}}{ImageNet-32} & \multicolumn{2}{@{}c@{}}{ImageNet-64} \\
        & & NFE & FID$\downarrow$ & NFE & FID$\downarrow$ \\
        \hline
        \multicolumn{4}{@{}l}{\textit{ODE/SDE-based methods}} \\
        St. Interpolants~\citep{albergo2023stochastic} & - & - & 8.49 & - & - \\
        FM $^\text{w}/$ OT~\citep{lipman2022flow} & dopri5 & 122 & 5.02 & 138 & 14.45 \\
        \hline
        \multicolumn{4}{@{}l}{\textit{Distilled models / Dedicated solvers}} \\
        \hline
        \multicolumn{4}{@{}l}{\textit{Improved training for straighter trajectories}} \\
        MFM $^\text{w}/$ BatchOT~\citep{pooladian2023multisample} & midpoint & 4 & 17.28 & 4 & 38.45 \\
        MFM $^\text{w}/$ BatchOT~\citep{pooladian2023multisample} & midpoint & 8 & 8.73 & 8 & 20.85 \\
        MFM $^\text{w}/$ BatchOT~\citep{pooladian2023multisample} & midpoint & 12 & 7.18 & 12 & 18.27 \\
        MFM $^\text{w}/$ Stable~\citep{pooladian2023multisample} & midpoint & 4 & 21.82 & 4 & 46.08 \\
        MFM $^\text{w}/$ Stable~\citep{pooladian2023multisample} & midpoint & 8 & 9.99 & 8 & 21.36 \\
        MFM $^\text{w}/$ Stable~\citep{pooladian2023multisample} & midpoint & 12 & 7.84 & 12 & 17.60 \\
        NFDM-OT~\citep{bartosh2024neural} & euler & 2 & 9.83 & 2 & 27.70 \\
        NFDM-OT~\citep{bartosh2024neural} & euler & 4 & 6.13 & 4 & 17.28 \\
        NFDM-OT~\citep{bartosh2024neural} & euler & 12 & 4.11 & 12 & 11.58 \\
        \hline
        \multicolumn{4}{@{}l}{\textit{Our results}} \\
        \methodName (\textit{ours}) 1 cache & midpoint & 4 & 13.47 & 4 & 37.24  \\
        \methodName (\textit{ours}) 1 cache & midpoint & 8 & 5.08 & 8 & 11.76 \\
        \methodName (\textit{ours}) 1 cache & midpoint & 12 & 3.89 & 12 & 8.49 \\
        \methodName (\textit{ours}) 1 cache & dopri5 & 137 & 2.75 & 133 & 6.63 \\
        % \methodName (\textit{ours}) 1 cache & midpoint & 4 & 13.36 & 4 & 37.24  \\
        % \methodName (\textit{ours}) 1 cache & midpoint & 8 & 5.07 & 8 & 11.76 \\
        % \methodName (\textit{ours}) 1 cache & midpoint & 12 & 3.85 & 12 & 8.49 \\
        % \methodName (\textit{ours}) 1 cache & dopri5 & 137 & 2.66 & 133 & 6.63 \\
        % \methodName (\textit{ours}) 1 cache & midpoint & 4 & 11.44 & 4 & 37.24  \\
        % \methodName (\textit{ours}) 1 cache & midpoint & 8 & 4.41 & 8 & 11.76 \\
        % \methodName (\textit{ours}) 1 cache & midpoint & 12 & 3.47 & 12 & 8.49 \\
        % \methodName (\textit{ours}) 1 cache & dopri5 & 137 & 2.88 & 133 & 6.63 \\
        \bottomrule
    \end{tabular}
\end{table}

\noindent\textbf{Rectified Flows.} As mentioned earlier, \methodName offers improved initialization for the \textit{reflow} algorithm (i.e., retraining the CFM using the coupling induced by the first-stage model). To validate this, we generated 1M noise-data pairs by sampling from our model trained with 4 caches on CIFAR10 and used those samples to train a new CFM model. The results, shown in Table~\ref{tab:cifar_results} (+ \textit{reflow}), indicate that \methodName combined with \textit{reflow} outperforms not only 2-Rectified Flow~\citep{liu2023flow} but also 3-Rectified Flow. This demonstrates that performing additional \textit{reflows} is unnecessary as long as the first \textit{reflow} is well-initialized.

\noindent\textbf{High-Resolution Image Synthesis.} Finally, we show that our method can be directly applied to training in the latent space of a pre-trained autoencoder, similar to the approaches in~\citet{rombach2022high, dao2023flow}. We trained \methodName on FFHQ $256\times 256$~\citep{ffhq} using a pre-trained autoencoder from~\citet{rombach2022high}. As shown in Table~\ref{tab:ffhq_results}, \methodName achieves a lower FID score with an order of magnitude fewer NFE compared to previous methods using the same model architecture. This highlights the compatibility of our method with conventional techniques for high-resolution synthesis and paves the way for exploring its application to high-resolution and large-scale datasets.

\begin{table}[t]
    \centering
    \footnotesize
    \caption{Unconditional generation results on FFHQ 256$\times$256}
    \label{tab:ffhq_results}
    \begin{tabular}{lccr}
        \toprule
        Method & Solver & \multicolumn{2}{@{}c@{}}{FFHQ-256} \\
        & & NFE & FID$\downarrow$ \\
        \hline
        \multicolumn{4}{@{}l}{\textit{Prior Work}} \\
        LDM~\citep{rombach2022high} & - & 50 & 4.98 \\
        ImageBART~\citep{esser2021imagebart} & - & 3 & 9.57 \\
        Latent Flow Matching~\citep{dao2023flow} (ADM) & dopri5 & 84 & 8.07 \\
        Latent Flow Matching~\citep{dao2023flow} (DiT L/2) & dopri5 & 88 & 4.55 \\
        \hline
        \multicolumn{4}{@{}l}{\textit{Our results}} \\
        \methodName (\textit{ours}) 4 caches (ADM) & midpoint & 2 & 14.89  \\
        \methodName (\textit{ours}) 4 caches (ADM) & midpoint & 4 & 6.76 \\
        \methodName (\textit{ours}) 4 caches (ADM) & midpoint & 8 & 5.50 \\
        \methodName (\textit{ours}) 4 caches (ADM) & dopri5 & 77 & 4.77 \\
        \bottomrule
    \end{tabular}
\end{table}

\section{Conclusions}

In this paper, we introduced \methodName, a method to improve data-noise coupling in training generative models with the CFM framework. Our method is based on finding locally optimal matchings between data and noise at each minibatch. In contrast to previous work, these matchings are stored and recycled in future iterations to obtain a better data-noise assignment at the global level. Through an extensive experimental section, we established that \methodName achieves a better sampling speed-quality trade-off than prior work with no additional computational cost and negligible disk usage overhead during training. \methodName is effective and simple enough to be composed with other techniques, such as rectified flows, to further enhance the sampling speed.

\section*{Acknowledgments}
This work has been supported by Swiss National Science Foundation Projects 188690 and 10001278.

\bibliography{references}

@article{ho2020denoising,
  title={Denoising diffusion probabilistic models},
  author={Ho, Jonathan and Jain, Ajay and Abbeel, Pieter},
  journal={Advances in neural information processing systems},
  volume={33},
  pages={6840--6851},
  year={2020}
}

@article{song2020denoising,
  title={Denoising diffusion implicit models},
  author={Song, Jiaming and Meng, Chenlin and Ermon, Stefano},
  journal={arXiv preprint arXiv:2010.02502},
  year={2020}
}

@article{tong2023conditional,
  title={Conditional flow matching: Simulation-free dynamic optimal transport},
  author={Tong, Alexander and Malkin, Nikolay and Huguet, Guillaume and Zhang, Yanlei and Rector-Brooks, Jarrid and Fatras, Kilian and Wolf, Guy and Bengio, Yoshua},
  journal={arXiv preprint arXiv:2302.00482},
  volume={2},
  number={3},
  year={2023}
}

@article{lipman2022flow,
  title={Flow matching for generative modeling},
  author={Lipman, Yaron and Chen, Ricky TQ and Ben-Hamu, Heli and Nickel, Maximilian and Le, Matt},
  journal={arXiv preprint arXiv:2210.02747},
  year={2022}
}

@inproceedings{albergo2023stochastic,
  title={Building Normalizing Flows with Stochastic Interpolants},
  author={Albergo, Michael and Vanden-Eijnden, Eric},
  booktitle={ICLR 2023 Conference},
  year={2023}
}

@article{pooladian2023multisample,
  title={Multisample flow matching: Straightening flows with minibatch couplings},
  author={Pooladian, Aram-Alexandre and Ben-Hamu, Heli and Domingo-Enrich, Carles and Amos, Brandon and Lipman, Yaron and Chen, Ricky TQ},
  journal={arXiv preprint arXiv:2304.14772},
  year={2023}
}

@article{bartosh2024neural,
  title={Neural Flow Diffusion Models: Learnable Forward Process for Improved Diffusion Modelling},
  author={Bartosh, Grigory and Vetrov, Dmitry and Naesseth, Christian A},
  journal={arXiv preprint arXiv:2404.12940},
  year={2024}
}

@inproceedings{lee2023minimizing,
  title={Minimizing trajectory curvature of ode-based generative models},
  author={Lee, Sangyun and Kim, Beomsu and Ye, Jong Chul},
  booktitle={International Conference on Machine Learning},
  pages={18957--18973},
  year={2023},
  organization={PMLR}
}

@article{dao2023flow,
  title={Flow matching in latent space},
  author={Dao, Quan and Phung, Hao and Nguyen, Binh and Tran, Anh},
  journal={arXiv preprint arXiv:2307.08698},
  year={2023}
}

@inproceedings{liu2023flow,
  title={Flow Straight and Fast: Learning to Generate and Transfer Data with Rectified Flow},
  author={Liu, Xingchao and Gong, Chengyue and Liu, Qiang},
  booktitle={The Eleventh International Conference on Learning Representations (ICLR)},
  year={2023}
}

@inproceedings{rombach2022high,
  title={High-resolution image synthesis with latent diffusion models},
  author={Rombach, Robin and Blattmann, Andreas and Lorenz, Dominik and Esser, Patrick and Ommer, Bj{\"o}rn},
  booktitle={Proceedings of the IEEE/CVF conference on computer vision and pattern recognition},
  pages={10684--10695},
  year={2022}
}

@article{esser2021imagebart,
  title={Imagebart: Bidirectional context with multinomial diffusion for autoregressive image synthesis},
  author={Esser, Patrick and Rombach, Robin and Blattmann, Andreas and Ommer, Bjorn},
  journal={Advances in neural information processing systems},
  volume={34},
  pages={3518--3532},
  year={2021}
}

@article{sommerfeld2019optimal,
  title={Optimal transport: Fast probabilistic approximation with exact solvers},
  author={Sommerfeld, Max and Schrieber, J{\"o}rn and Zemel, Yoav and Munk, Axel},
  journal={Journal of Machine Learning Research},
  volume={20},
  number={105},
  pages={1--23},
  year={2019}
}

@article{schrieber2016dotmark,
  title={Dotmark--a benchmark for discrete optimal transport},
  author={Schrieber, J{\"o}rn and Schuhmacher, Dominic and Gottschlich, Carsten},
  journal={IEEE Access},
  volume={5},
  pages={271--282},
  year={2016},
  publisher={IEEE}
}

@article{dhariwal2021diffusion,
  title={Diffusion models beat gans on image synthesis},
  author={Dhariwal, Prafulla and Nichol, Alexander},
  journal={Advances in neural information processing systems},
  volume={34},
  pages={8780--8794},
  year={2021}
}

@inproceedings{ronneberger2015u,
  title={U-net: Convolutional networks for biomedical image segmentation},
  author={Ronneberger, Olaf and Fischer, Philipp and Brox, Thomas},
  booktitle={Medical image computing and computer-assisted intervention--MICCAI 2015: 18th international conference, Munich, Germany, October 5-9, 2015, proceedings, part III 18},
  pages={234--241},
  year={2015},
  organization={Springer}
}

@article{kuhn1955hungarian,
  title={The Hungarian method for the assignment problem},
  author={Kuhn, Harold W},
  journal={Naval research logistics quarterly},
  volume={2},
  number={1-2},
  pages={83--97},
  year={1955},
  publisher={Wiley Online Library}
}

@article{sinkhorn,
author = {Knight, Philip A.},
title = {The Sinkhorn–Knopp Algorithm: Convergence and Applications},
journal = {SIAM Journal on Matrix Analysis and Applications},
volume = {30},
number = {1},
pages = {261-275},
year = {2008},
doi = {10.1137/060659624},
URL = { 
        https://doi.org/10.1137/060659624
},
eprint = { 
        https://doi.org/10.1137/060659624
}
}

@book{villani2009optimal,
  title={Optimal transport: old and new},
  author={Villani, C{\'e}dric and others},
  volume={338},
  year={2009},
  publisher={Springer}
}

@article{fatras2021minibatch,
  title={Minibatch optimal transport distances; analysis and applications},
  author={Fatras, Kilian and Zine, Younes and Majewski, Szymon and Flamary, R{\'e}mi and Gribonval, R{\'e}mi and Courty, Nicolas},
  journal={arXiv preprint arXiv:2101.01792},
  year={2021}
}

@article{cuturi2013sinkhorn,
  title={Sinkhorn distances: Lightspeed computation of optimal transport},
  author={Cuturi, Marco},
  journal={Advances in neural information processing systems},
  volume={26},
  year={2013}
}

@article{xie2024randomized,
  title={Randomized methods for computing optimal transport without regularization and their convergence analysis},
  author={Xie, Yue and Wang, Zhongjian and Zhang, Zhiwen},
  journal={Journal of Scientific Computing},
  volume={100},
  number={2},
  pages={37},
  year={2024},
  publisher={Springer}
}

@article{peyre2019computational,
  title={Computational optimal transport: With applications to data science},
  author={Peyr{\'e}, Gabriel and Cuturi, Marco and others},
  journal={Foundations and Trends{\textregistered} in Machine Learning},
  volume={11},
  number={5-6},
  pages={355--607},
  year={2019},
  publisher={Now Publishers, Inc.}
}

@misc{roughgarden2016cs261,
  title={CS261: A Second Course in Algorithms Lecture\# 5: Minimum-Cost Bipartite Matching},
  author={Roughgarden, Tim},
  year={2016},
  publisher={January}
}

@inproceedings{
shaul2024bespoke,
title={Bespoke Solvers for Generative Flow Models},
author={Neta Shaul and Juan Perez and Ricky T. Q. Chen and Ali Thabet and Albert Pumarola and Yaron Lipman},
booktitle={The Twelfth International Conference on Learning Representations},
year={2024},
url={https://openreview.net/forum?id=1PXEY7ofFX}
}

@inproceedings{
song2021denoising,
title={Denoising Diffusion Implicit Models},
author={Jiaming Song and Chenlin Meng and Stefano Ermon},
booktitle={International Conference on Learning Representations},
year={2021},
url={https://openreview.net/forum?id=St1giarCHLP}
}

@article{lu2022dpm,
  title={Dpm-solver: A fast ode solver for diffusion probabilistic model sampling in around 10 steps},
  author={Lu, Cheng and Zhou, Yuhao and Bao, Fan and Chen, Jianfei and Li, Chongxuan and Zhu, Jun},
  journal={Advances in Neural Information Processing Systems},
  volume={35},
  pages={5775--5787},
  year={2022}
}

@article{song2023consistency,
  title={Consistency models},
  author={Song, Yang and Dhariwal, Prafulla and Chen, Mark and Sutskever, Ilya},
  journal={arXiv preprint arXiv:2303.01469},
  year={2023}
}

@article{luhman2021knowledge,
  title={Knowledge distillation in iterative generative models for improved sampling speed},
  author={Luhman, Eric and Luhman, Troy},
  journal={arXiv preprint arXiv:2101.02388},
  year={2021}
}

@inproceedings{esser2024scaling,
  title={Scaling rectified flow transformers for high-resolution image synthesis},
  author={Esser, Patrick and Kulal, Sumith and Blattmann, Andreas and Entezari, Rahim and M{\"u}ller, Jonas and Saini, Harry and Levi, Yam and Lorenz, Dominik and Sauer, Axel and Boesel, Frederic and others},
  booktitle={Forty-first International Conference on Machine Learning},
  year={2024}
}

@article{lee2024improving,
  title={Improving the Training of Rectified Flows},
  author={Lee, Sangyun and Lin, Zinan and Fanti, Giulia},
  journal={arXiv preprint arXiv:2405.20320},
  year={2024}
}

@inproceedings{
salimans2022progressive,
title={Progressive Distillation for Fast Sampling of Diffusion Models},
author={Tim Salimans and Jonathan Ho},
booktitle={International Conference on Learning Representations},
year={2022},
url={https://openreview.net/forum?id=TIdIXIpzhoI}
}

@article{goodfellow2020generative,
  title={Generative adversarial networks},
  author={Goodfellow, Ian and Pouget-Abadie, Jean and Mirza, Mehdi and Xu, Bing and Warde-Farley, David and Ozair, Sherjil and Courville, Aaron and Bengio, Yoshua},
  journal={Communications of the ACM},
  volume={63},
  number={11},
  pages={139--144},
  year={2020},
  publisher={ACM New York, NY, USA}
}

@article{fid,
  title={Gans trained by a two time-scale update rule converge to a local nash equilibrium},
  author={Heusel, Martin and Ramsauer, Hubert and Unterthiner, Thomas and Nessler, Bernhard and Hochreiter, Sepp},
  journal={Advances in neural information processing systems},
  volume={30},
  year={2017}
}

@article{cifar10,
  title={Learning multiple layers of features from tiny images},
  author={Krizhevsky, Alex and Hinton, Geoffrey and others},
  year={2009},
  publisher={Toronto, ON, Canada}
}

@article{imagenet,
Author = {Olga Russakovsky and Jia Deng and Hao Su and Jonathan Krause and Sanjeev Satheesh and Sean Ma and Zhiheng Huang and Andrej Karpathy and Aditya Khosla and Michael Bernstein and Alexander C. Berg and Li Fei-Fei},
Title = {{ImageNet Large Scale Visual Recognition Challenge}},
Year = {2015},
journal   = {International Journal of Computer Vision (IJCV)},
doi = {10.1007/s11263-015-0816-y},
volume={115},
number={3},
pages={211-252}
}

@inproceedings{ffhq,
  title={A style-based generator architecture for generative adversarial networks},
  author={Karras, Tero and Laine, Samuli and Aila, Timo},
  booktitle={Proceedings of the IEEE/CVF conference on computer vision and pattern recognition},
  pages={4401--4410},
  year={2019}
}

@inproceedings{davtyan2023efficient,
  title={Efficient video prediction via sparsely conditioned flow matching},
  author={Davtyan, Aram and Sameni, Sepehr and Favaro, Paolo},
  booktitle={Proceedings of the IEEE/CVF International Conference on Computer Vision},
  pages={23263--23274},
  year={2023}
}

@article{hu2023motion,
  title={Motion flow matching for human motion synthesis and editing},
  author={Hu, Vincent Tao and Yin, Wenzhe and Ma, Pingchuan and Chen, Yunlu and Fernando, Basura and Asano, Yuki M and Gavves, Efstratios and Mettes, Pascal and Ommer, Bjorn and Snoek, Cees GM},
  journal={arXiv preprint arXiv:2312.08895},
  year={2023}
}

@article{guan2024lafma,
  title={LAFMA: A Latent Flow Matching Model for Text-to-Audio Generation},
  author={Guan, Wenhao and Wang, Kaidi and Zhou, Wangjin and Wang, Yang and Deng, Feng and Wang, Hui and Li, Lin and Hong, Qingyang and Qin, Yong},
  journal={arXiv preprint arXiv:2406.08203},
  year={2024}
}

@article{ho2022classifier,
  title={Classifier-free diffusion guidance},
  author={Ho, Jonathan and Salimans, Tim},
  journal={arXiv preprint arXiv:2207.12598},
  year={2022}
}
\bibliographystyle{iclr2025_conference}

\newpage
\appendix

\section*{Appendix}

Here we provide some additional details and results that could not be included in the main paper due to the page number limit. Those include: the derivation of the optimal vector field for the Gaussian to Gaussian case (Appendix~\ref{app:gtg}); the proof of the finite convergence of our algorithm (Appendix~\ref{app:analysis}); implementation details (Appendix~\ref{app:impl}); more quantitative results (Appendix~\ref{app:detailed_results}); more qualitative results (Appendix~\ref{app:qualitative_results}).

\section{Gaussian to Gaussian Case}\label{app:gtg}
In this section, we are looking for the closed form solution for the optimal flow in the case of both source and target distributions being standard normal distributions. That is, we are aiming to finding $\hat v(y, t)$ such that:
\begin{align}
    \hat v(y, t) = \arg \min_{v(y, t)} \mathbb{E}_{x, z \sim {\cal N}(0, 1), t \sim U[0, 1], \varepsilon \sim {\cal N}(0, \sigma^2)} \left\| v(t x + (1 - t) z + \varepsilon, t) - (x - z) \right\|^2, \label{eq:gtg_cfm}
\end{align}

If we write down the expectation explicitly through integrals, we obtain
\begin{align}
    \int_{\mathbb{R}^d} \int_{\mathbb{R}^d} \int_0^1 \int_{\mathbb{R}^d} \left\| v(y, t) - (x - z) \right\|^2 {\cal N}(y | t x + (1 - t) z, \sigma^2 I) {\cal N}(z | 0, I) {\cal N}(x | 0, I) \; dy dt dx dz
\end{align}

Let us define

\begin{align}
    {\cal L}(y, t, v) = \int_{\mathbb{R}^d} \int_{\mathbb{R}^d} \left\| v(y, t) - (x - z) \right\|^2 {\cal N}(y | t x + (1 - t) z, \sigma^2 I) \; dx dz.
\end{align}

The above functional in the objective takes the form

\begin{align}
    \int_{\mathbb{R}^d} \int_0^1 {\cal L}(y, t, v(y, t)) \; dy dt.
\end{align}

In order to find the optimal $v(y, t)$, one needs to write down the Euler-Lagrange equation (notice that ${\cal L}$ does not depend on the derivatives of $v$, and hence the Euler-Lagrange equation only includes the derivative w.r.t $v$)

\begin{align}
    \frac{\partial {\cal L}(y, t, v(y, t))}{\partial v} = 0,
\end{align}

which in this special case takes the form (here and later on we omit the integration domains for saving space and assume $\mathbb{R}^d$ everywhere)

\begin{align}
\iint 2 * \left(v(y, t) - (x - z)\right) {\cal N}(y | t x + (1 - t) z, \sigma^2 I) {\cal N}(z | 0, I) {\cal N}(x | 0, I) \; dx dz = 0.
\end{align}

And hence, solving for $v(y, t)$, gives

\begin{align}
    \hat v(y, t) = \frac{\iint (x - z) {\cal N}(y | t x + (1 - t) z, \sigma^2 I) {\cal N}(z | 0, I) {\cal N}(x | 0, I) \; dx dz}{\iint {\cal N}(y | t x + (1 - t) z, \sigma^2 I) {\cal N}(z | 0, I) {\cal N}(x | 0, I) \; dx dz}
\end{align}

First, let us calculate the denominator:

{\footnotesize
\begin{gather}
\iint \mathcal{N}(y \mid t x + (1 - t) z, \sigma^2 I) \mathcal{N}(z \mid 0, I) \mathcal{N}(x \mid 0, I) \, dx \, dz = \\
\left| C_0 = \frac{1}{(2\pi)^{\frac{3d}{2}}\sigma^{d}} \right| \\
= C_0 \iint \exp\left(-\frac{1}{2}\left(x^\top x + z^\top z + \frac{1}{\sigma^2} 
\left((y - tx) - (1 - t)z\right)^\top\left((y - tx) - (1 - t)z\right)\right)\right) dx \, dz \\
= C_0 \iint \exp\left(-\frac{1}{2}\left(1 + \frac{(1 - t)^2}{\sigma^2}\right) z^\top z + \frac{(1 - t)}{\sigma^2}(y - tx)^\top z - \frac{1}{2}x^\top x - \frac{1}{2\sigma^2} (y - tx)^\top(y - tx)\right) dx \, dz \\
\left| C_1 = \frac{1}{(2\pi)^d (\sigma^2 + (1 - t)^2)^{\frac{d}{2}}} \right| \\
= C_1 \int \exp\left(-\frac{1}{2} x^\top x - \frac{1}{2}\frac{1}{\sigma^2 + (1 - t)^2} \left[y^\top y + t^2 x^\top x - 2t y^\top x\right] \right) dx \\
= C_1 \int \exp\left(-\frac{1}{2}\left( 1 + \frac{t^2}{\sigma^2 + (1 - t)^2}\right) x^\top x + \frac{t}{\sigma^2 + (1 - t)^2} y^\top x - \frac{1}{2} \frac{1}{\sigma^2 + (1 - t)^2} y^\top y\right) dx \\
= C_1 \frac{(2\pi)^{\frac{d}{2}}}{\left(1 + \frac{t^2}{\sigma^2 + (1 - t)^2}\right)^\frac{d}{2}} \exp \left(\frac{1}{2}\left(\frac{t}{\sigma^2 + (1 - t)^2}\right)^2 \frac{1}{1 + \frac{t^2}{\sigma^2 + (1 - t)^2}} y^\top y - \frac{1}{2}\frac{1}{\sigma^2 + (1 - t)^2} y^\top y\right) \\
\end{gather}
} 

Here we used the well-known Gaussian integrals of the form

\begin{align}
    \int \exp(-\frac{1}{2} x^\top A x + b^\top x + c) \,dx = \sqrt{\det(2\pi A^{-1})} \exp(\frac{1}{2}b^\top A^{-1} b + c).
\end{align}

Finally, the last expression can be simplified to

\begin{gather}
    (2\pi)^{-\frac{d}{2}} (\sigma^2 + t^2 + (1 - t)^2)^{-\frac{d}{2}} \exp \left(-\frac{1}{2} \frac{y^\top y}{\sigma^2 + t^2 + (1 - t)^2}\right),
\end{gather}

which is the probability density function of

\begin{align}
    {\cal N}(y \,| \,0, \left(\sigma^2 + t^2 + (1 - t)^2\right) I).
\end{align}

Now let us move to the nominator, which can be rewritten as $F(y, t) - F(y, 1 - t)$, where

\begin{align}
    F(y, t) = \iint x \,{\cal N}(y | t x + (1 - t) z, \sigma^2 I) {\cal N}(z | 0, I) {\cal N}(x | 0, I) \; dx dz.
\end{align}

Thus, we only need to calculate $F(y, t)$. Similarly to the denominator, using the well-known Gaussian integrals of the form 

\begin{align}
    \int x \cdot \exp\left(-\frac{1}{2} x^\top x + b^\top x + c\right) \, dx = \left(\frac{2\pi}{a}\right)^{\frac{d}{2}} \frac{b}{a} \exp\left(\frac{b^\top b}{2 a} + c\right),
\end{align}

we can obtain

\begin{align}
    F(y, t) = \frac{yt}{\sigma^2 + t^2 + (1 - t)^2} {\cal N}(y \, | \, 0, \sigma^2 + t^2 + (1 - t)^2).
\end{align}

Hence, by combining all the results, we get the optimal vector field

\begin{align}
    \hat v(y, t) = y \cdot s(t), \quad \text{where}\quad s(t) = \frac{2t - 1}{\sigma^2 + t^2 + (1 - t)^2}.
\end{align}

The vector field of such form is parallel to the line connecting $y$ and the origin and changes its direction halfway through time $t \in [0, 1]$ (see \figref{fig:st}). This analytical result coincides with the empirical observation of fitting the vector field as a neural network via gradient descent in the two-dimensional setting (see \secref{sec:cfm}).

\begin{figure}
    \centering
    \includegraphics[width=0.5\linewidth]{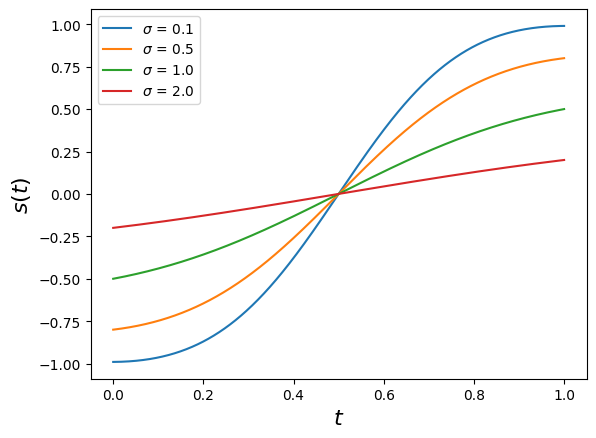}
    \caption{The form of the multiplier $s(t)$ by $y$ in the optimal vector field $v(y, t) = y \cdot s(t)$ for different values of $\sigma$.}
    \label{fig:st}
\end{figure}

\section{Analysis}
\label{app:analysis}
In this section we provide a proof for the theorem established in section \ref{sec:lbm} and include a counter-example, where only using batch size equal to the size of the dataset could solve for the globally optimal matching (see Fig.~\ref{fig:counter}).

As noted in section \ref{sec:ot}, \ref{eq:assignment_perm} can be understood as a minimum cost perfect matching problem in a complete bipartite graph. The assignment task is captured by the graph $\mathcal{G} = (V = X \cup Z, E)$, where $X$ is the set of nodes representing the data points and $Z$ is the set of nodes representing the noise samples. The set of edges $E = X\times E$, where the cost of an edge $(x_i, z_j) \in E$ is given by the $c_{ij} = \|x_i - z_j\|_2$, defines the complete bipartite graph (see Figure \ref{fig:loom} for a visual representation).

A perfect matching $M$ is a set of $n$ edges in $\mathcal{G}$ that cover all the vertices. The cost of a matching is the sum of the costs of the edges. An optimal matching is one achieving the minimal cost. The optimality of a matching $M$ is entirely captured by the absence of \emph{negative $M$-alternating cycles}.

\begin{definition}[M-alternating cycles]
    Given a perfect matching $M$, any cycle in $\mathcal{G}$ given by a sequence of edges $C = (x_{i_1}, z_{i_1})(z_{i_1}, x_{i_2})(x_{i_2}, z_{i_2}) \dots (x_{i_k}, z_{i_k})(z_{i_k}, x_{i_1})$ where for any $p \in [k]$, we have $(x_{i_p}, z_{i_p}) \in M$ is called an alternating $M$-cycle.
\end{definition}

$M$-alternating cycles as their name suggests, pass from $X$ to $Z$ back and forth while passing through edges in $M$ to go from $X$ to $Z$. An $M$-alternating cycle is said to be negative if the sum of the costs of the edges in $M$ is larger than the cost of the edges in absent from $M$:
\begin{equation}
\sum_{e \in C\setminus M} c_e < \sum_{e \in C \cap M} c_e.
\label{eq:negativecycle}
\end{equation}
Consequently, for any matching $M$, if a negative $M$-alternating cycle $C$ exists, then matching $M'$ with lower cost can be constructed by taking $M' = (C\setminus M) \cup (M\setminus C)$. In other words, by swapping the \emph{going edges} in $C$(from $X$ to $Z$) with the \emph{return edges} (from $Z$ to $X$), we can diminish the cost thanks to \eqref{eq:negativecycle}. A necessary condition for optimality is thus the absence of negative $M$-alternating cycles. The following theorem shows that it is in fact a sufficient condition.

\begin{proposition}[Thm 2.2 \citep{roughgarden2016cs261}]
    A matching $M$ is optimal if and only if there are no negative $M$-alternating cycles.
    \label{prop:globopt}
\end{proposition}

Given the exposition above, \methodName can be understood as an iterative negative-cycle elimination scheme. Indeed the algorithm starts from a matching $M_0$ given by the edges $(x_i, z_{\tau_0(i)})$, then at each iteration $k$, a subgraph $\mathcal{G}_k = (V_k = \{x_{n_j}\}_{j=1}^m \cup \{z_{\tau_{k - 1}(n_j)}\}_{j=1}^{m}, E_k)$ consisting of only the $2m$ vertices sampled at \texttt{step} $3$ in \algref{alg:loom} is considered. An optimal matching is computed within that subgraph and the matching $M_{k+1}$ is obtained by updating the edges involving vertices in $V_k$. As computing optimal matchings are equivalent to elimination of all negative cycles, the update of $M_k$ corresponds to the elimination of negative $M_k$-alternating cycles that are contained in $\mathcal{G}_k$. Such cycles are of length at most $m$. The process described by algorithm \ref{alg:loom} is thus a stepwise elimination of negative $M_k$-alternating cycles of length at most $m$. This allows us to characterize the convergence of our method in the following theorem.

\convergencethm*
\begin{proof}
    Let $c(\tau) := \sum_{i \in \mathbb{N}_n} c(x_i, z_{\tau(i)})$. Each update of \methodName (step 4 and 5 of \algref{alg:loom}) improves the matching over a sub-graph selected at \texttt{step 3}. The sequence $c(\tau_k)_k$ is thus non-increasing. Moreover, for all $k$, the sequence of costs takes its values in the finite set $\{ c(\tau): \; \tau \in  S_n \}$. It then follows that $(c(\tau_k))_k$ becomes constant after a finite number of iterations $K_{\text{final}}$ since it is a monotone sequence taking values in a finite set.

    Suppose now that there existed a negative alternating cycle of length less than $m$ in the matching induced by $\tau_{K_{\text{final}}}$. At each iteration, a subgraph containing a negative cycle has a non-zero probability of being selected at \texttt{step 3}. Since $(c(\tau_k))_k$ is constant for all $k > K_{\text{final}}$, this implies that such a subgraph is never sampled at \texttt{step 3} for all $k > K_{\text{final}}$. The probability of such a sub-graph never being sampled for all $k > K_{\text{final}}$ is $0$ by a simple application of Borell-Cantelli's lemma. Consequently, the probability of $\tau_{K_{\text{final}}}$ containing a negative alternating cycle is $0$.
\end{proof}

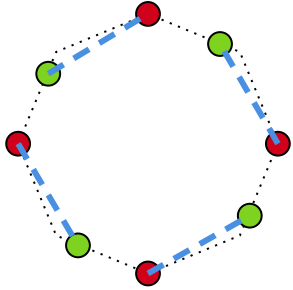
\begin{figure}
    \centering
    \makebox[0.5\textwidth]{
    \tikzset{every picture/.style={line width=0.75pt}} %set default line width to 0.75pt        

\begin{tikzpicture}[x=0.75pt,y=0.75pt,yscale=-1,xscale=1]

%Shape: Regular Polygon [id:dp09304593463357025] 
\draw  [dash pattern={on 0.84pt off 2.51pt}] (396,146.5) -- (376.82,192.82) -- (330.5,212) -- (284.18,192.82) -- (265,146.5) -- (284.18,100.18) -- (330.5,81) -- (376.82,100.18) -- cycle ;
%Shape: Circle [id:dp48655289932510937] 
\draw  [fill={rgb, 255:red, 208; green, 2; blue, 27 }  ,fill opacity=1 ] (259,146.5) .. controls (259,143.19) and (261.69,140.5) .. (265,140.5) .. controls (268.31,140.5) and (271,143.19) .. (271,146.5) .. controls (271,149.81) and (268.31,152.5) .. (265,152.5) .. controls (261.69,152.5) and (259,149.81) .. (259,146.5) -- cycle ;
%Shape: Circle [id:dp0524479279549952] 
\draw  [fill={rgb, 255:red, 208; green, 2; blue, 27 }  ,fill opacity=1 ] (324.5,212) .. controls (324.5,208.69) and (327.19,206) .. (330.5,206) .. controls (333.81,206) and (336.5,208.69) .. (336.5,212) .. controls (336.5,215.31) and (333.81,218) .. (330.5,218) .. controls (327.19,218) and (324.5,215.31) .. (324.5,212) -- cycle ;
%Shape: Circle [id:dp5223863430117186] 
\draw  [fill={rgb, 255:red, 208; green, 2; blue, 27 }  ,fill opacity=1 ] (390,146.5) .. controls (390,143.19) and (392.69,140.5) .. (396,140.5) .. controls (399.31,140.5) and (402,143.19) .. (402,146.5) .. controls (402,149.81) and (399.31,152.5) .. (396,152.5) .. controls (392.69,152.5) and (390,149.81) .. (390,146.5) -- cycle ;
%Shape: Circle [id:dp06809182280290682] 
\draw  [fill={rgb, 255:red, 208; green, 2; blue, 27 }  ,fill opacity=1 ] (324.5,81) .. controls (324.5,77.69) and (327.19,75) .. (330.5,75) .. controls (333.81,75) and (336.5,77.69) .. (336.5,81) .. controls (336.5,84.31) and (333.81,87) .. (330.5,87) .. controls (327.19,87) and (324.5,84.31) .. (324.5,81) -- cycle ;
%Shape: Circle [id:dp5608850217102097] 
\draw  [fill={rgb, 255:red, 126; green, 211; blue, 33 }  ,fill opacity=1 ] (274.18,111.18) .. controls (274.18,107.87) and (276.87,105.18) .. (280.18,105.18) .. controls (283.5,105.18) and (286.18,107.87) .. (286.18,111.18) .. controls (286.18,114.5) and (283.5,117.18) .. (280.18,117.18) .. controls (276.87,117.18) and (274.18,114.5) .. (274.18,111.18) -- cycle ;
%Shape: Circle [id:dp4094213520446309] 
\draw  [fill={rgb, 255:red, 126; green, 211; blue, 33 }  ,fill opacity=1 ] (289.18,197.82) .. controls (289.18,194.5) and (291.87,191.82) .. (295.18,191.82) .. controls (298.5,191.82) and (301.18,194.5) .. (301.18,197.82) .. controls (301.18,201.13) and (298.5,203.82) .. (295.18,203.82) .. controls (291.87,203.82) and (289.18,201.13) .. (289.18,197.82) -- cycle ;
%Shape: Circle [id:dp13994880098846596] 
\draw  [fill={rgb, 255:red, 126; green, 211; blue, 33 }  ,fill opacity=1 ] (375.82,182.82) .. controls (375.82,179.5) and (378.5,176.82) .. (381.82,176.82) .. controls (385.13,176.82) and (387.82,179.5) .. (387.82,182.82) .. controls (387.82,186.13) and (385.13,188.82) .. (381.82,188.82) .. controls (378.5,188.82) and (375.82,186.13) .. (375.82,182.82) -- cycle ;
%Shape: Circle [id:dp8175278935370389] 
\draw  [fill={rgb, 255:red, 126; green, 211; blue, 33 }  ,fill opacity=1 ] (360.82,96.18) .. controls (360.82,92.87) and (363.5,90.18) .. (366.82,90.18) .. controls (370.13,90.18) and (372.82,92.87) .. (372.82,96.18) .. controls (372.82,99.5) and (370.13,102.18) .. (366.82,102.18) .. controls (363.5,102.18) and (360.82,99.5) .. (360.82,96.18) -- cycle ;
%Straight Lines [id:da1336979283095251] 
\draw [color={rgb, 255:red, 74; green, 144; blue, 226 }  ,draw opacity=1 ][line width=2.25]  [dash pattern={on 6.75pt off 4.5pt}]  (280.18,111.18) -- (330.5,81) ;
%Straight Lines [id:da9856735611138749] 
\draw [color={rgb, 255:red, 74; green, 144; blue, 226 }  ,draw opacity=1 ][line width=2.25]  [dash pattern={on 6.75pt off 4.5pt}]  (265,146.5) -- (295.18,197.82) ;
%Straight Lines [id:da6508316946938765] 
\draw [color={rgb, 255:red, 74; green, 144; blue, 226 }  ,draw opacity=1 ][line width=2.25]  [dash pattern={on 6.75pt off 4.5pt}]  (396,146.5) -- (366.82,96.18) ;
%Straight Lines [id:da16668406222239884] 
\draw [color={rgb, 255:red, 74; green, 144; blue, 226 }  ,draw opacity=1 ][line width=2.25]  [dash pattern={on 6.75pt off 4.5pt}]  (330.5,212) -- (381.82,182.82) ;

\end{tikzpicture}
    }
    \caption{Counter-example showing the existence of stationary sub-optimal matchings. The data points (red) are placed at every other vertex of a regular polygon, the noise points (green) are slightly offset in a counterclockwise fashion from the vertices of the polygon. The optimal assignment matches each data points to the following noise point going in the clockwise direction. The assignment represented in blue is sub-optimal but it is stationary for any batch size $m < n$. }
    \label{fig:counter}
\end{figure}

Although this theorem shows that \methodName converges to a stationary solution, it does not guarantee global optimality. However, as already mentioned in \secref{sec:lbm} of the main paper, this is the case for all minibatch-based methods. We exhibit the counter-example of \citet{xie2024randomized} (see \figref{fig:counter}) showing that any sub-problem-based approach would fail to recover the optimal coupling as a full cyclic permutation of all the points is necessary to solve the problem. Despite this, the coupling induced by \methodName is more optimal than that of minibatch OT-based methods and leads to straighter sampling trajectories, as evidenced by the experiments on real data.

\section{Implementation Details}\label{app:impl}
This section includes the implementation details of our algorithm, including model architectures, training parameters, etc.

In practice, to be comparable with the prior work, we parametrize the learned vector field with a neural network that has an improved UNet~\citep{ronneberger2015u} architecture (ADM) from \citet{dhariwal2021diffusion}. To make fair comparisons, we used the same model architectures and training parameters as in the prior work, when possible. See Table~\ref{tab:adm_params} for the network architecture configurations and training parameters. All models (except for the ablations) were trained on 4 Nvidia RTX 3090 GPUs. For the latent space models we used the pretrained autoencoder from Stable Diffusion~\citep{rombach2022high} provided at the Hugging Face model hub\footnote{\url{https://huggingface.co/stabilityai/sd-vae-ft-mse-original/blob/main/vae-ft-mse-840000-ema-pruned.ckpt}}. The training code can be found at \url{https://github.com/araachie/loom-cfm}.

\begin{table}[t]
    \centering
    \caption{ADM network architecture and training parameters of \methodName for each model.}\label{tab:adm_params}
    \begin{tabular}{lcccc}
         \toprule
         & CIFAR10 & ImageNet-32 & ImageNet-64 & FFHQ256 \\
         \hline
         Input shape & [3, 32, 32] & [3, 32, 32] & [3, 64, 64] & [4, 32, 32] \\
         Channels & 128 & 128 & 192 & 256 \\
         Number of Res blocks & 2 & 3 & 2 & 2 \\
         Channels multipliers & [1, 2, 2, 2] & [1, 2, 2, 2] & [1, 2, 3, 4] & [1, 2, 3, 4] \\
         Heads & 4 & 4 & 4 & 4 \\
         Heads channels & 64 & 64 & 64 & 64 \\
         Attention resolution & [16] & [16, 8] & [16] & [16, 8, 4] \\
         Dropout & 0.1 & 0.1 & 0.1 & 0.1 \\
         \hline
         Effective batch size & 128 & 512 & 96 & 128 \\
         GPUs & 4 & 4 & 4 & 4 \\
         Epochs &  1000 & 200 & 100 & 500 \\
         Iterations & 391k & 500k & 1334k & 273k \\
         Learning rate & 0.0002 & 0.0001 & 0.0001 & 0.00002 \\
         Learning rate scheduler & Constant & Constant & Constant & Constant \\
         Warmup steps & 5k & 20k & 20k & 3.5k \\
         EMA decay & 0.9999 & 0.9999 & 0.9999 & 0.9999 \\
         Training time (hours) & 17.3 & 73.5 & 190.6 & 66.8 \\
         \hline
         CFM $\sigma$ & 1e-7 & 1e-7 & 1e-7 & 1e-7 \\
         Number of noise caches & 4 & 1 & 1 & 4 \\
         \bottomrule
    \end{tabular}
\end{table}

\section{Quantitative Results}\label{app:detailed_results}
In this section, we provide more quantitative evaluations of our method.

\noindent\textbf{Ablations.} For a 
better visual perception, the results of ablations were presented as plots in the main paper. Here we include the numbers used to build those plots (see Table~\ref{tab:cifar_ablations}).

\noindent\textbf{Convergence.} Throughout the training we log the number of reassignments per minibatch. We report those in \figref{fig:num_swaps_bs} for different batch sizes and in \figref{fig:num_swaps_nc} for different number of noise caches (mimicking different dataset sizes). Additionally, we also report the minibatch OT cost around the training time where the reassignments are quite rare (\twofigref{fig:cost_bs}{fig:cost_nc}). It can be seen that \methodName with different batch sizes despite starting at different number of reassignments (around the batch size, since the intial assignments are random) converges roughly in the same amount of time (see \figref{fig:num_swaps_bs}). However, with larger batch size the method converges to a better matching, as can be inferred from the minibatch OT cost (see \figref{fig:cost_bs}). At the same time, increasing the dataset size leads to slower convergence, since the algorithm has to visit more minibatches to sort out the assignments (see \figref{fig:num_swaps_nc}). And the minibatch OT cost behaves accordingly (see \figref{fig:cost_nc}).

\noindent\textbf{Training time.} As for the time complexity, \methodName is comparable to OT-CFM~\citep{tong2023conditional} or BatchOT~\citep{pooladian2023multisample} as it solves the same minibatch matching problem as the prior methods. However, saving and loading the current assignments from the disk introduces a small i/o overhead. Hence, the training time of \methodName is slightly longer than that of the baselines. For instance, training for 1000 epochs on CIFAR10 with batch size 128 scattered across 4 Nvidia RTX 3090 GPUs takes 17.3 hours with \methodName compared to 13.9 hours when the assignments are not stored.

\begin{minipage}{.5\linewidth}%
    \centering
    \footnotesize
    \captionof{table}{Detailed ablation results on CIFAR10.}
    \label{tab:cifar_ablations}
    \begin{tabular}{ccr}
        \toprule
        Solver & NFE & FID$\downarrow$ \\
        \hline
        \multicolumn{3}{c}{\textit{Fixed source, w/o reassignments}} \\
        \hline
        midpoint & 4 & 22.57 \\
        midpoint & 8 & 7.61 \\
        midpoint & 12 & 6.17 \\
        dopri5 & 142 & 5.98 \\
        % w/o reassignments, fixed source & midpoint & 4 & 22.31 \\
        % w/o reassignments, fixed source & midpoint & 8 & 7.22 \\
        % w/o reassignments, fixed source & midpoint & 12 & 5.93 \\
        % w/o reassignments, fixed source & dopri5 & 142 & 5.73 \\
        \hline
        \multicolumn{3}{c}{\textit{W/o saving local reassignments}} \\
        \hline
        midpoint & 4 & 21.47 \\ % 4 & 22.02 \\
        midpoint & 8 & 8.69 \\ % 8 & 9.27 \\
        midpoint & 12 & 6.56 \\ % 12 & 7.26 \\
        dopri5 & 131 & 4.52 \\ % 136 & 5.26 \\
        % w/o saving reassignments, new source & midpoint & 4 & 21.25 \\ % 4 & 22.02 \\
        % w/o saving reassignments, new source & midpoint & 8 & 8.30 \\ % 8 & 9.27 \\
        % w/o saving reassignments, new source & midpoint & 12 & 6.21 \\ % 12 & 7.26 \\
        % w/o saving reassignments, new source & dopri5 & 131 & 4.27 \\ % 136 & 5.26 \\
        \hline
        \multicolumn{3}{c}{\textit{Train after convergence}} \\
        \hline
        midpoint & 4 & 12.62 \\ % 4 & 22.02 \\
        midpoint & 8 & 6.45 \\ % 8 & 9.27 \\
        midpoint & 12 & 5.53 \\ % 12 & 7.26 \\
        dopri5 & 139 & 5.07 \\ % 136 & 5.26 \\
        % train after & midpoint & 4 & 12.36 \\ % 4 & 22.02 \\
        % train after & midpoint & 8 & 6.11 \\ % 8 & 9.27 \\
        % train after & midpoint & 12 & 5.21 \\ % 12 & 7.26 \\
        % train after & dopri5 & 139 & 4.73 \\ % 136 & 5.26 \\
        \hline
        % w/o saving reassignments, new source, bs 32 & midpoint & 4 & 26.23 \\
        % w/o saving reassignments, new source, bs 32 & midpoint & 8 & 10.89 \\
        % w/o saving reassignments, new source, bs 32 & midpoint & 12 & 8.14 \\
        % w/o saving reassignments, new source, bs 32 & dopri5 & 131 & 5.34 \\
        % w/o saving reassignments, new source, bs 64 & midpoint & 4 & 29.05 \\
        % w/o saving reassignments, new source, bs 64 & midpoint & 8 & 11.65 \\
        % w/o saving reassignments, new source, bs 64 & midpoint & 12 & 8.81 \\
        % w/o saving reassignments, new source, bs 64 & dopri5 & 123 & 5.72 \\
        % 1 cache & midpoint & 4 & 23.53 \\
        % 1 cache & midpoint & 8 & 28.51 \\
        % 1 cache & midpoint & 12 & 30.27 \\
        % 1 cache & dopri5 & - & - \\
        \multicolumn{3}{c}{\textit{2 caches, bs 128}} \\
        \hline
        midpoint & 4 & 12.28 \\
        midpoint & 8 & 9.93 \\
        midpoint & 12 & 9.81 \\
        dopri5 & 126 & 9.75 \\
        % 2 caches & midpoint & 4 & 11.86 \\
        % 2 caches & midpoint & 8 & 9.70 \\
        % 2 caches & midpoint & 12 & 9.50 \\
        % 2 caches & dopri5 & 126 & 9.51 \\
        \hline
        \multicolumn{3}{c}{\textit{3 caches, bs 128}} \\
        \hline
        midpoint & 4 & 11.30 \\
        midpoint & 8 & 6.40 \\
        midpoint & 12 & 5.92 \\
        dopri5 & 134 & 5.82 \\ % 138 & 6.38 \\
        % 3 caches & midpoint & 4 & 10.98 \\
        % 3 caches & midpoint & 8 & 6.21 \\
        % 3 caches & midpoint & 12 & 5.57 \\
        % 3 caches & dopri5 & 138 & 5.51 \\ % 138 & 6.38 \\
        \hline
        \multicolumn{3}{c}{\textit{4 caches, bs 128}} \\
        \hline
        midpoint & 4 & 11.60 \\ % & 4 & 12.17 \\
        midpoint & 8 & 5.38 \\ % 8 & 6.16 \\
        midpoint & 12 & 4.60 \\ % 12 & 5.42 \\
        dopri5 & 134 & 4.41 \\
        % 4 caches, bs 128 & midpoint & 4 & 11.06 \\ % & 4 & 12.17 \\
        % 4 caches, bs 128 & midpoint & 8 & 4.95 \\ % 8 & 6.16 \\
        % 4 caches, bs 128 & midpoint & 12 & 4.29 \\ % 12 & 5.42 \\
        % 4 caches, bs 128 & dopri5 & 134 & 4.17 \\ % 135 & 5.23 \\
        \hline
        \multicolumn{3}{c}{\textit{4 caches, bs 64}} \\
        \hline
        midpoint & 4 & 19.39 \\
        midpoint & 8 & 8.85 \\
        midpoint & 12 & 6.68 \\
        dopri5 & 134 & 4.51 \\
        % 4 caches, bs 64 & midpoint & 4 & 18.61 \\
        % 4 caches, bs 64 & midpoint & 8 & 8.51 \\
        % 4 caches, bs 64 & midpoint & 12 & 6.37 \\
        % 4 caches, bs 64 & dopri5 & 134 & 4.17 \\
        \hline
        \multicolumn{3}{c}{\textit{4 caches, bs 32}} \\
        \hline
        midpoint & 4 & 20.83 \\
        midpoint & 8 & 9.56 \\
        midpoint & 12 & 7.40 \\
        dopri5 & 122 & 5.02 \\
        % 4 caches, bs 32 & midpoint & 4 & 20.63 \\
        % 4 caches, bs 32 & midpoint & 8 & 9.16 \\
        % 4 caches, bs 32 & midpoint & 12 & 7.01 \\
        % 4 caches, bs 32 & dopri5 & 122 & 4.61 \\
        \bottomrule
    \end{tabular}
\end{minipage}%
\hfill%
\begin{minipage}{.5\linewidth}
    \centering
    \includegraphics[width=\linewidth]{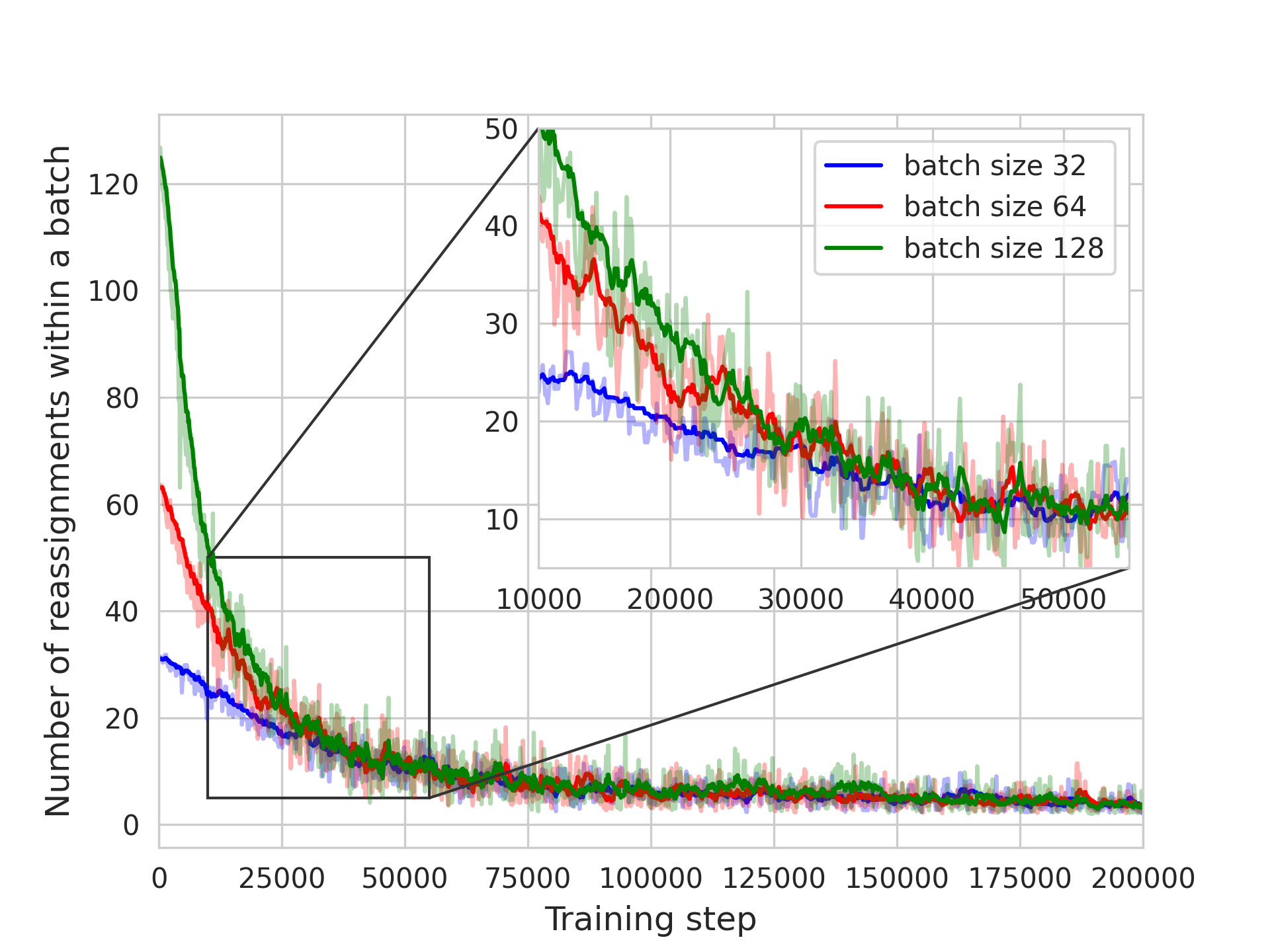}
    \captionof{figure}{Number of reassignments within a minibatch along the training progress depending on the batch size.}
    \label{fig:num_swaps_bs}
    \includegraphics[width=\linewidth]{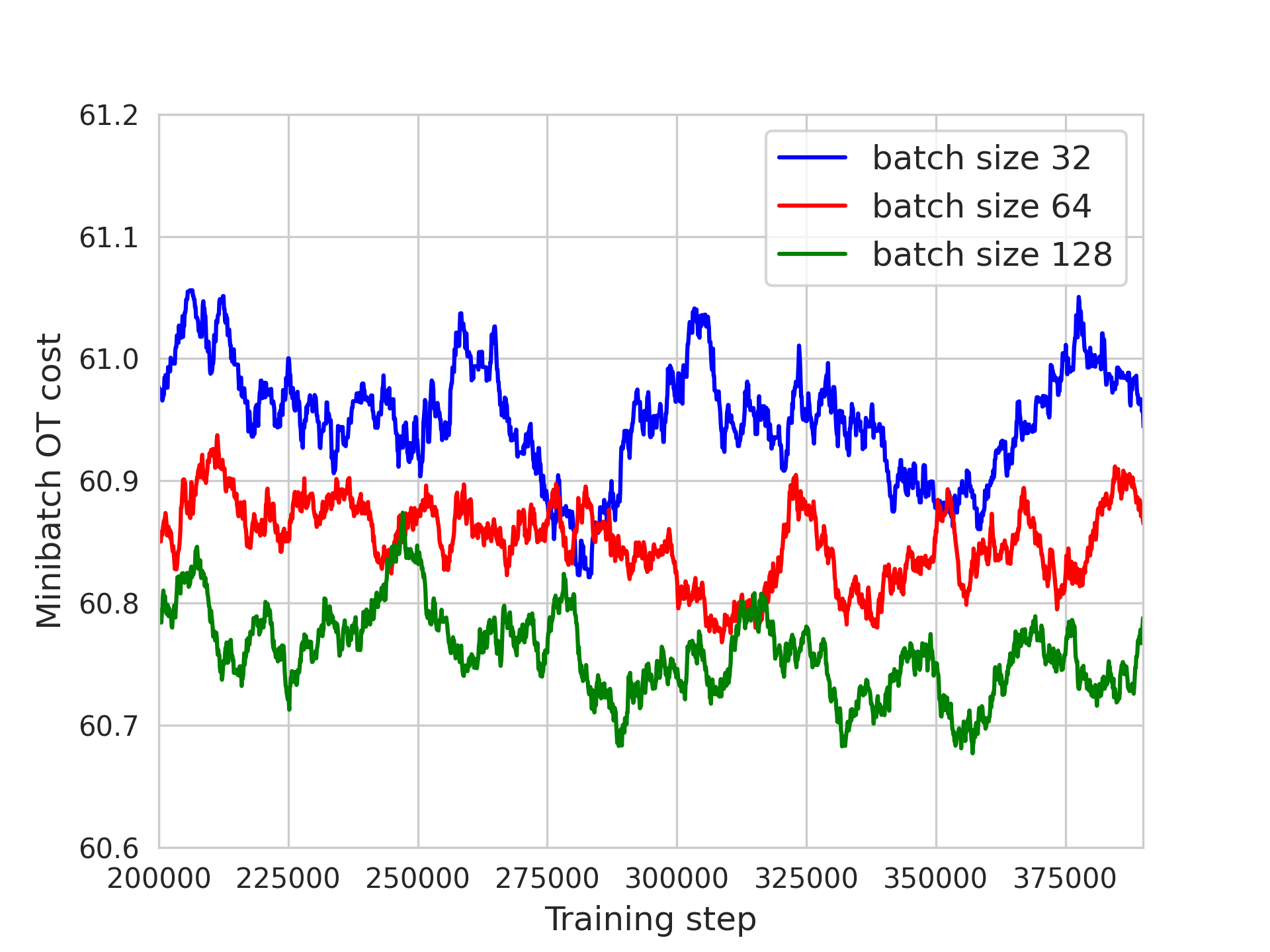}
    \captionof{figure}{OT cost per minibatch along the training progress depending on the batch size. The lines corresponding to the exponential moving average of the cost are shown.}
    \label{fig:cost_bs}
\end{minipage}

\begin{minipage}{.49\linewidth}
    \centering
    \includegraphics[width=\linewidth]{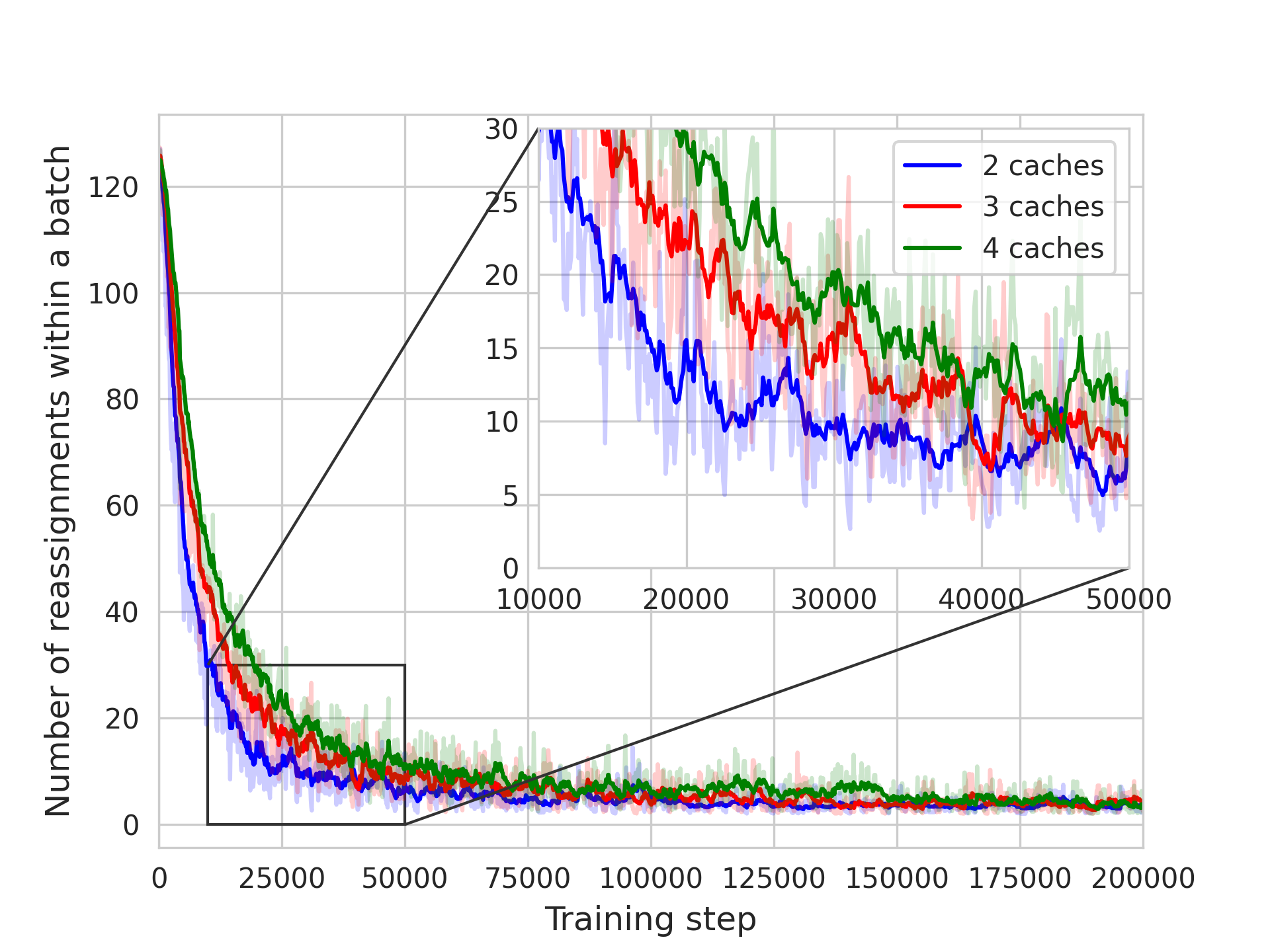}
    \captionof{figure}{Number of reassignments within a minibatch along the training progress depending on the number of noise caches (the size of the dataset). The lines corresponding to the exponential moving average of the cost are shown.}
    \label{fig:num_swaps_nc}
\end{minipage}%
\hfill%
\begin{minipage}{.49\linewidth}
    \centering
    \includegraphics[width=\linewidth]{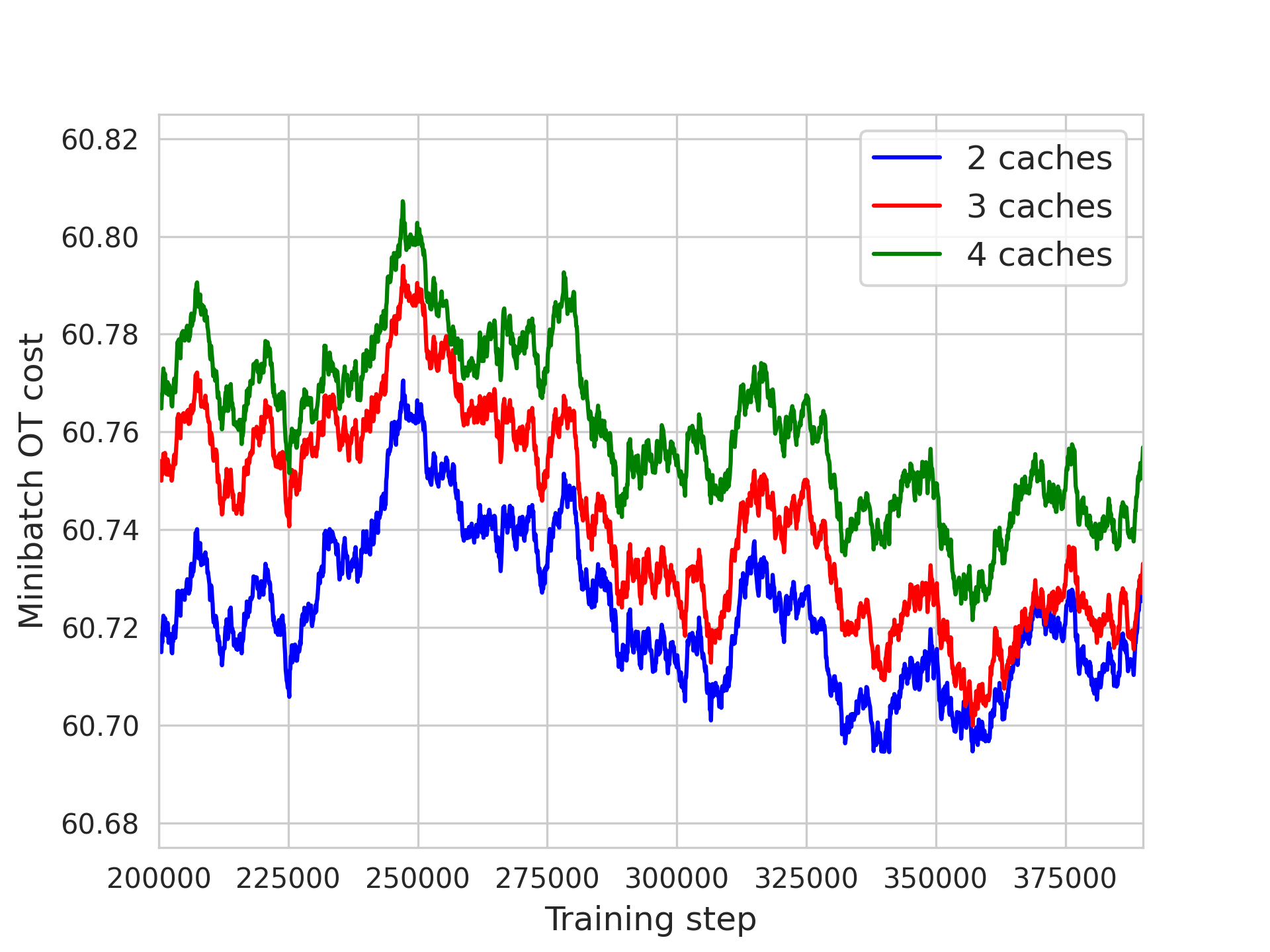}
    \captionof{figure}{OT cost per minibatch along the training progress depending on the number of noise caches (the size of the dataset). The lines corresponding to the exponential moving average of the cost are shown.}
    \label{fig:cost_nc}
\end{minipage}

\section{Qualitative Results}
\label{app:qualitative_results}
In this section, we provide some qualitative results that could not be included in the main paper due to the page number limit. 

\noindent\textbf{Sampling Paths.}
We provide examples of sampling trajectories from our models. Starting from a randomly sampled noise instance at $t = 0$, our models iteratively denoise those via numerically integrating the ODE in \eqref{eq:ode} to obtain clean images. In \quadfigref{fig:st_cifar}{fig:st_in32}{fig:st_in64}{fig:st_ffhq} we show intermediate points visited by the solver at $t \leq 1$. For the model trained on CIFAR10, we also compare the sampling paths to the version of the model that does not store reassignments. For FFHQ, in \figref{fig:st_nfe_ffhq} we additionally provide samples using different number of steps in the ODE solver, resulting in different NFEs.

\noindent\textbf{Unconditional Generation.} Lastly, we provide uncurated samples from our models in \quadfigref{fig:unc_cifar}{fig:unc_in32}{fig:unc_in64}{fig:unc_ffhq}.

\begin{figure}[t]
    \centering
    \begin{tabular}{c}
         \includegraphics[width=0.8\linewidth]{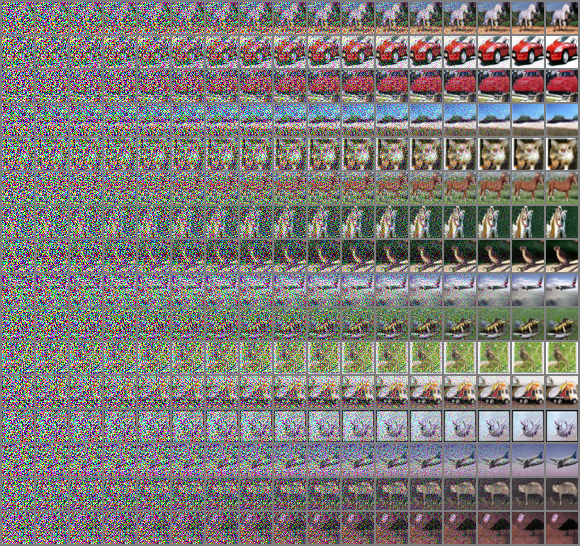} \\
         \includegraphics[width=0.8\linewidth]{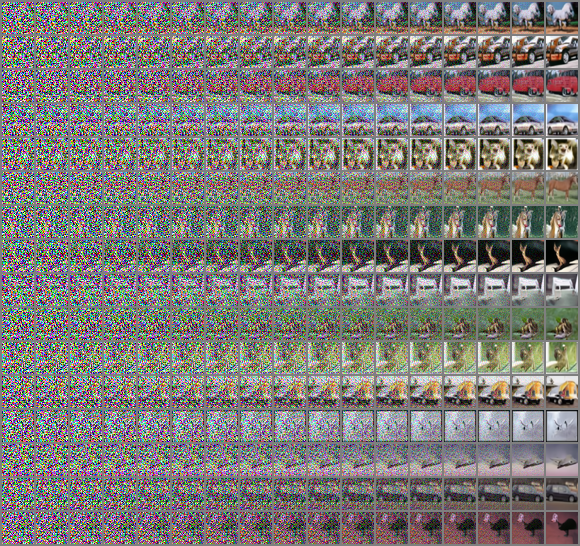}
    \end{tabular}
    \caption{Sampling trajectories of \methodName (top) and the version of the algorithm that does not save the reassignments (bottom) trained on CIFAR10. The $i$-th row in both grids starts with the same initial noise. \methodName tends to produce sharper genereations and converges earlier.}
    \label{fig:st_cifar}
\end{figure}

\begin{figure}[t]
    \centering
    \includegraphics[width=0.8\linewidth]{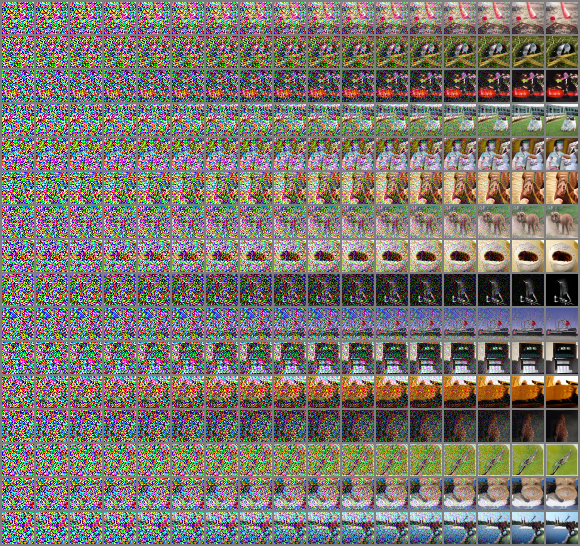}
    \caption{Sampling trajectories of the model trained on ImageNet-32.}
    \label{fig:st_in32}
\end{figure}

\begin{figure}[t]
    \centering
    \includegraphics[width=0.8\linewidth]{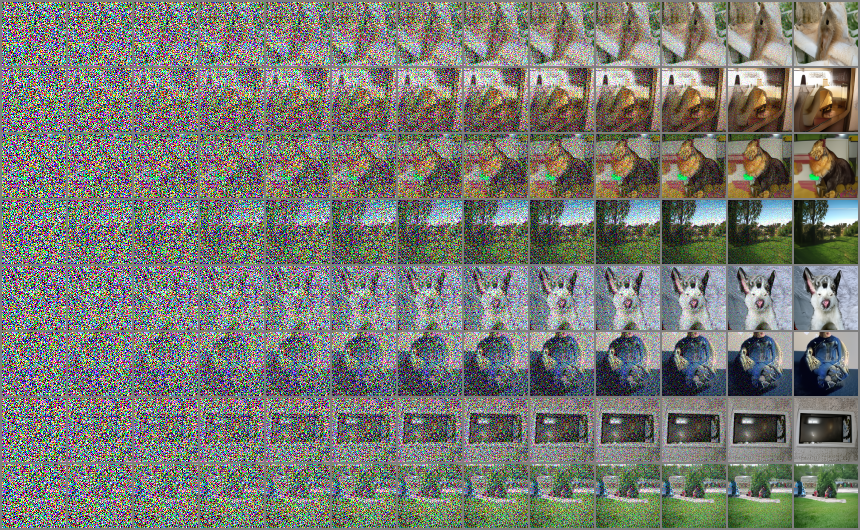}
    \caption{Sampling trajectories of the model trained on ImageNet-64.}
    \label{fig:st_in64}
\end{figure}

\begin{figure}[t]
    \centering
    \includegraphics[width=0.7\linewidth]{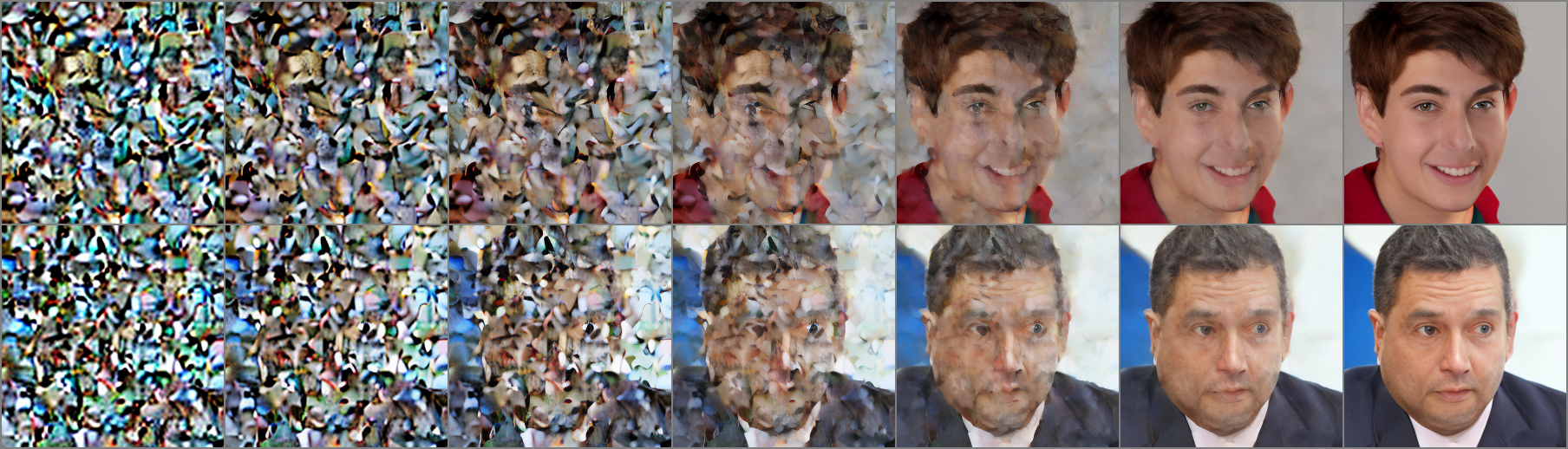}
    \caption{Sampling trajectories of the model trained on FFHQ-256.}
    \label{fig:st_ffhq}
\end{figure}

\begin{figure}[t]
    \centering
    \includegraphics[width=0.7\linewidth]{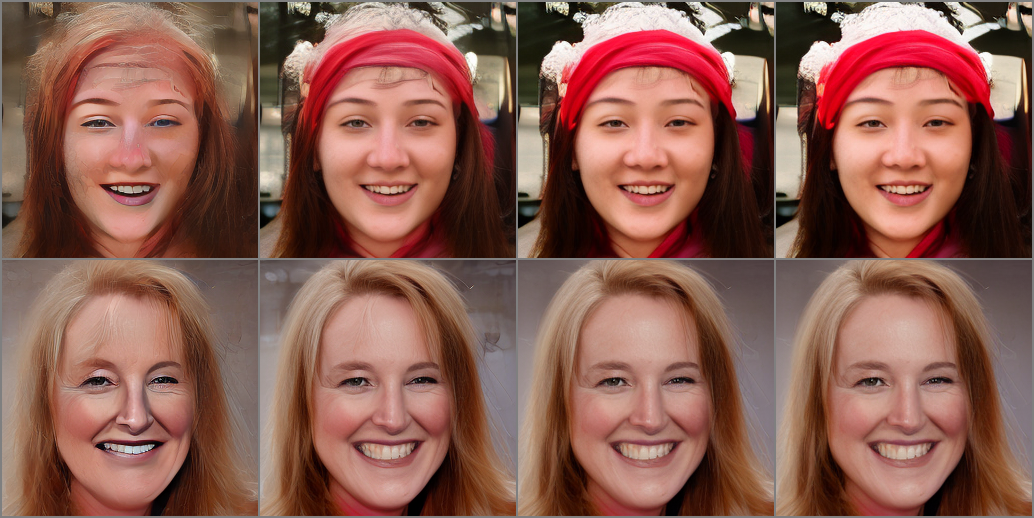}
    \caption{Samples from the FFHQ-256 model with different NFE (from left to right: 2, 4, 8 and 12 function evaluations).}
    \label{fig:st_nfe_ffhq}
\end{figure}

\begin{figure}[t]
    \centering
    \includegraphics[width=0.7\linewidth]{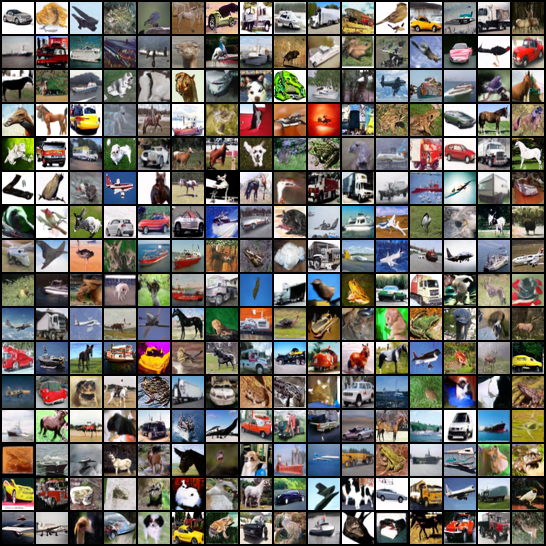}
    \caption{Uncurated samples from the model trained on CIFAR10.}
    \label{fig:unc_cifar}
\end{figure}

\begin{figure}[t]
    \centering
    \includegraphics[width=0.8\linewidth]{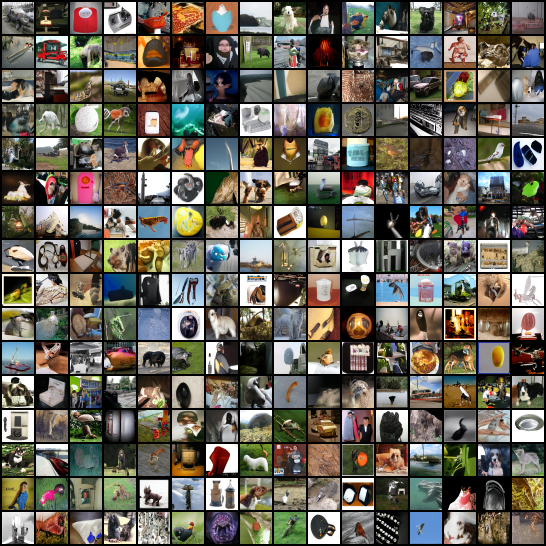}
    \caption{Uncurated samples from the model trained on ImageNet-32.}
    \label{fig:unc_in32}
\end{figure}

\begin{figure}[t]
    \centering
    \includegraphics[width=0.8\linewidth]{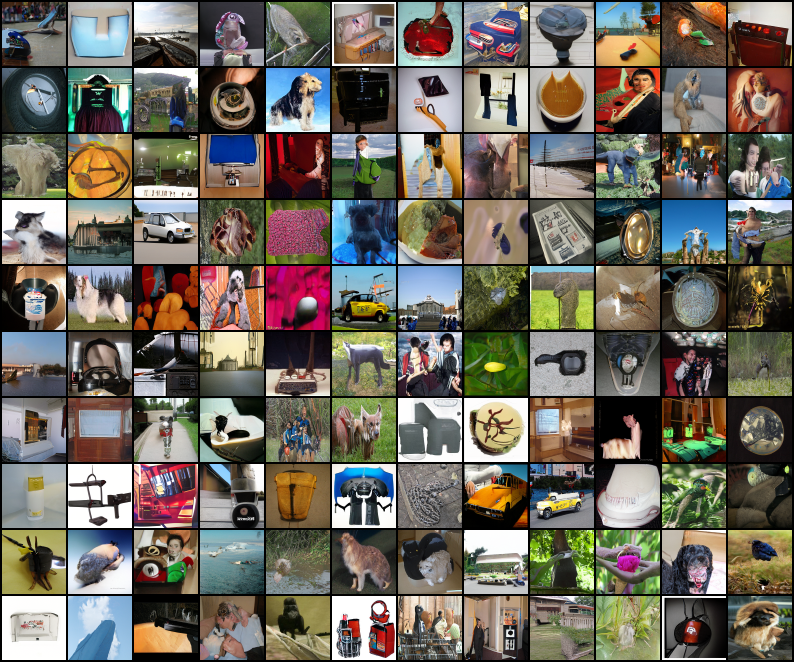}
    \caption{Uncurated samples from the model trained on ImageNet-64.}
    \label{fig:unc_in64}
\end{figure}

\begin{figure}[t]
    \centering
    \includegraphics[width=0.8\linewidth]{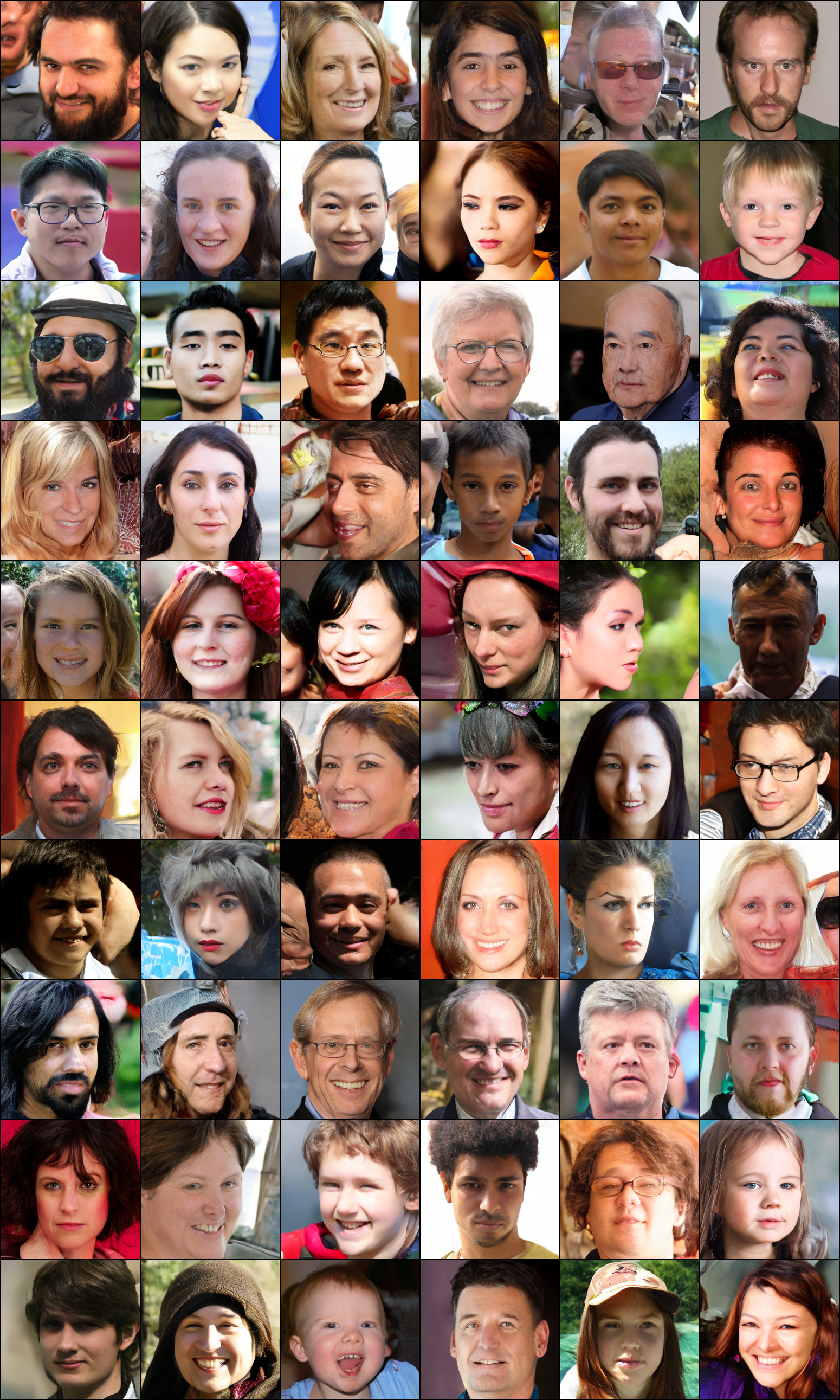}
    \caption{Uncurated samples from the model trained on FFHQ-256.}
    \label{fig:unc_ffhq}
\end{figure}

\section{Alternative Approaches}
\label{app:alternative_approaches}
In the early stages of developing \methodName, we explored various approaches to prevent overfitting to a fixed set of noise samples. Although none of these approaches yielded satisfactory outcomes, we summarize them here for clarity and completeness.

\begin{enumerate}
    \item \textbf{Completely refreshing the noise.} We experimented with replacing all the noise samples every $N$ epochs. While intuitive, this approach caused training instability since the newly assigned noise samples did not necessarily align with the previous ones. Consequently, the network’s targets could shift significantly when the noise was refreshed, impairing convergence.
    \item \textbf{Gradual noise injection.} In this approach, we introduced a hyperparameter, $\phi$, to control the noise refreshing. For each data point in a minibatch, the assigned noise was replaced with a new sample with probability $\phi$. Although this method allowed for a smoother refresh, choosing an optimal $\phi$ proved to be challenging. For example, setting $\phi = 0.1$ caused approximately one third of each minibatch to be reshuffled, which weakened the effect of LOOM-CFM. In contrast, a lower value of $\phi=0.01$ was insufficient to prevent overfitting.
    \item \textbf{Interpolation between LOOM-CFM and independent coupling.} For each data point, we coupled it with its assigned noise with probability $\phi$ and with freshly sampled noise with probability $1 - \phi$. Unlike approach 2, the new noise did not replace the cached noise. We found this technique to be quite effective and the results indeed interpolated the results of LOOM-CFM and the independent coupling. We also explored making $\phi$ a function of $t$, as the curvature of sampling paths depends on $t$ (as shown in \figref{fig:abl_curv}). Unfortunately, this hasn't lead to substantial improvements. Nevertheless, we found this approach interesting for its generality as it can interpolate any pair of coupling techniques. We leave further exploration of this method for future work.
\end{enumerate}

In contrast to these more complex methods, the approach with multiple noise caches in \methodName is straightforward, as it artificially increases the dataset size and equalizes the settings for problems with small and large dataset sizes.

\section{Limitations}
\label{app:limitations}
One limitation of \methodName and other OT-based methods is that they are not directly compatible with conditional generation, especially when the conditioning signal is complex.  In any conditional setup, the marginals of all label-conditional probability paths at $t = 0$ must match the source distribution. However, a naive implementation of any coupling-based method, by conditioning the model without adapting the couplings, may introduce sampling bias, as certain labels could disproportionately align with specific regions in the noise space. Although techniques like classifier(-free) guidance~\citep{ho2022classifier} could be adapted, we leave the exploration of those extensions to future work.

\end{document}